\documentclass[conference]{IEEEtran}
\usepackage{times}

\usepackage[numbers]{natbib}
\usepackage{multicol}
\usepackage[bookmarks=true]{hyperref}

\let\labelindent\relax

\usepackage[numbers]{natbib}
\usepackage{multicol}
\usepackage[bookmarks=true]{hyperref}
\usepackage{amsmath}
\usepackage{amssymb}
\usepackage{color}
\usepackage{graphicx}
\usepackage{appendix}
\graphicspath{{images/}}

\usepackage{xspace}

\usepackage{setspace}

\usepackage{amsmath}
\usepackage{amsfonts}
\usepackage{amssymb}
\usepackage{amsthm}
\usepackage{bm}
\usepackage{bbm}
\usepackage{mathtools}
\usepackage{enumitem}
\usepackage{thmtools,thm-restate}
\usepackage{algorithm}
\usepackage{algorithmic}
\usepackage{color}
\usepackage{graphicx}
\usepackage{comment}

\def\Rbb{\mathbb{R}}

\def\R{\Rbb}

\def\*{\star}

\newcommand{\tr}[1]{ \mathrm{tr}\left( #1\right)}

\usepackage[dvipsnames]{xcolor}

\renewcommand{\tr}{{T}}

\newcommand{\paral}{{\!/\mkern-5mu/\!}}

\newcommand{\Lag}{\mathcal{L}}
\newcommand{\Ham}{\mathcal{H}}

\newcommand{\ma}{\mathbf{a}}
\newcommand{\p}{\mathbf{p}}
\newcommand{\q}{\mathbf{q}}
\newcommand{\qd}{{\dot{\q}}}
\newcommand{\qdd}{{\ddot{\q}}}
\newcommand{\uu}{\mathbf{u}}

\newcommand{\vv}{\mathbf{v}}

\newcommand{\x}{\mathbf{x}}
\newcommand{\xd}{{\dot{\x}}}
\newcommand{\xdd}{{\ddot{\x}}}
\newcommand{\y}{\mathbf{y}}

\newcommand{\z}{\mathbf{z}}

\newcommand{\f}{\mathbf{f}}
\newcommand{\h}{\mathbf{h}}
\newcommand{\g}{\mathbf{g}}

\newcommand{\zero}{\mathbf{0}}

\newcommand{\J}{\mathbf{J}}
\newcommand{\Jd}{{\dot{\J}}}

\newcommand{\A}{\mathbf{A}}
\newcommand{\B}{\mathbf{B}}

\newcommand{\D}{\mathbf{D}}

\newcommand{\mH}{\mathbf{H}}
\newcommand{\I}{\mathbf{I}}

\newcommand{\M}{\mathbf{M}}

\newcommand{\mP}{\mathbf{P}}
\newcommand{\mR}{\mathbf{R}}

\newcommand{\V}{\mathbf{V}}

\newcommand{\wt}[1]{{\widetilde{#1}}}

\usepackage{amsmath}  
\usepackage{amssymb}
\usepackage{accents}  
\usepackage{amsthm}

\theoremstyle{plain}
\newtheorem{theorem}{Theorem}[section]
\newtheorem{lemma}[theorem]{Lemma}
\newtheorem{proposition}[theorem]{Proposition}
\newtheorem{corollary}[theorem]{Corollary}

\theoremstyle{definition}
\newtheorem{definition}[theorem]{Definition}

\theoremstyle{remark}
\newtheorem{remark}[theorem]{Remark}

\makeatletter
\let\save@mathaccent\mathaccent
\newcommand*\if@single[3]{%
  \setbox0\hbox{${\mathaccent"0362{#1}}^H$}%
  \setbox2\hbox{${\mathaccent"0362{\kern0pt#1}}^H$}%
  \ifdim\ht0=\ht2 #3\else #2\fi
  }
\newcommand*\rel@kern[1]{\kern#1\dimexpr\macc@kerna}
\newcommand*\widebar[1]{\@ifnextchar^{{\wide@bar{#1}{0}}}{\wide@bar{#1}{1}}}
\newcommand*\wide@bar[2]{\if@single{#1}{\wide@bar@{#1}{#2}{1}}{\wide@bar@{#1}{#2}{2}}}
\newcommand*\wide@bar@[3]{%
  \begingroup
  \def\mathaccent##1##2{%
    \let\mathaccent\save@mathaccent
    \if#32 \let\macc@nucleus\first@char \fi
    \setbox\z@\hbox{$\macc@style{\macc@nucleus}_{}$}%
    \setbox\tw@\hbox{$\macc@style{\macc@nucleus}{}_{}$}%
    \dimen@\wd\tw@
    \advance\dimen@-\wd\z@
    \divide\dimen@ 3
    \@tempdima\wd\tw@
    \advance\@tempdima-\scriptspace
    \divide\@tempdima 10
    \advance\dimen@-\@tempdima
    \ifdim\dimen@>\z@ \dimen@0pt\fi
    \rel@kern{0.6}\kern-\dimen@
    \if#31
      \overline{\rel@kern{-0.6}\kern\dimen@\macc@nucleus\rel@kern{0.4}\kern\dimen@}%
      \advance\dimen@0.4\dimexpr\macc@kerna
      \let\final@kern#2%
      \ifdim\dimen@<\z@ \let\final@kern1\fi
      \if\final@kern1 \kern-\dimen@\fi
    \else
      \overline{\rel@kern{-0.6}\kern\dimen@#1}%
    \fi
  }%
  \macc@depth\@ne
  \let\math@bgroup\@empty \let\math@egroup\macc@set@skewchar
  \mathsurround\z@ \frozen@everymath{\mathgroup\macc@group\relax}%
  \macc@set@skewchar\relax
  \let\mathaccentV\macc@nested@a
  \if#31
    \macc@nested@a\relax111{#1}%
  \else
    \def\gobble@till@marker##1\endmarker{}%
    \futurelet\first@char\gobble@till@marker#1\endmarker
    \ifcat\noexpand\first@char A\else
      \def\first@char{}%
    \fi
    \macc@nested@a\relax111{\first@char}%
  \fi
  \endgroup
}
\makeatother

\usepackage{comment}

\newcommand{\fullversiononly}{\iffalse}
\newcommand{\compressedversiononly}{\iftrue}

\renewcommand{\baselinestretch}{.995}

\title{Optimization Fabrics for Behavioral Design}

\author{\authorblockN{Nathan D. Ratliff\authorrefmark{1},
Karl Van Wyk\authorrefmark{1},
Mandy Xie\authorrefmark{1}\authorrefmark{2}, 
Anqi Li\authorrefmark{1}\authorrefmark{3} and
Muhammad Asif Rana\authorrefmark{2}}
\authorblockA{
\authorrefmark{1}NVIDIA (\{nratliff,kvanwyk\}@nvidia.com);
\authorrefmark{2}Georgia Tech;
\authorrefmark{3}University of Washington}
}

\begin{document}

\maketitle
\begin{abstract}
A common approach to the provably stable design of reactive behavior, exemplified by operational space control, is to reduce the problem to the design of virtual classical mechanical systems (energy shaping). This framework is widely used, and through it we gain stability, but at the price of expressivity. This work presents a comprehensive theoretical framework expanding this approach showing that there is a much larger class of differential equations generalizing classical mechanical systems (and the broader class of Lagrangian systems) and greatly expanding their expressivity while maintaining the same governing stability principles. At the core of our framework is a class of differential equations we call fabrics which constitute a behavioral medium across which we can optimize a potential function. These fabrics shape the system’s behavior during optimization but still always provably converge to a local minimum, making them a building block of stable behavioral design. We build the theoretical foundations of our framework here and provide a simple empirical demonstration of a practical class of {\em geometric fabrics}, which additionally exhibit a natural geometric path consistency making them convenient for flexible and intuitive behavioral design.
\end{abstract}

\IEEEpeerreviewmaketitle

\section{Introduction}

Riemannian Motion Policies (RMPs) \cite{ratliff2018rmps} are powerful tools for modularizing robotic behavior. They enable designers to construct behavior in parts, designing behavioral components separately in different task spaces
and then combining them using the RMP algebra within the RMPflow algorithm \cite{cheng2018rmpflow}. While demonstrably a powerful design framework (see \cite{ratliff2018rmps,cheng2018rmpflow}), the most effective class of RMPs \cite{ratliff2018rmps} lack formal stability guarantees and can be challenging for inexperienced designers to use. 

This paper addresses the gap between theory and practice by modeling behavior with {\em optimization fabrics}, or {\em fabrics} for short. Fabrics are a special form of second-order differential equation defining task-independent nominal behaviors which influence the dynamics of a continuous time optimizer. Objective functions define tasks (with local minima defining task goals), and the fabric shapes the optimization path of the system as it descends the potential en route to the goal. Behavior arises from the natural dynamics of the nominal behavior encoded into the fabric, and the forcing potential defining the task objectives. We characterize optimization fabrics with increasing levels of specificity, building toward a pragmatic and flexible form of fabric called the {\em geometric fabric}.

We present the theory of optimization fabrics in its broadest generalization, so this exposition is highly theoretical. The theory captures the most important intuitions behind RMP design, including the ability to practice modular {\em acceleration-based} design, within a well-understood and provably stable framework that overcomes a number of the restrictions and limitations of past attempts (see Subsection~\ref{sec:RelatedWork}).

Our theory greatly generalizes the tools underpinning classical reactive control techniques such as operational space control \cite{Peters_AR_2008}. Operational space control, and its geometric control generalization \cite{bullo2004geometric}, use the physical properties of virtual classical mechanical systems within an {\em energy shaping} paradigm to attain provable stability. In this work, we show that these classical mechanical systems (equivalently Riemannian geometries) are a form of {\em optimization fabric}, which we characterize broadly and concretely as differential equations which optimize a potential function when forced by its negative gradient. 

Within this formalism, it is easy to see that we need not restrict ourselves to classical mechanical systems. As a first step, we expand our tool set to include the more general class of Lagrangian systems \cite{LeonardSusskindClassicalMechanics}. We discuss a particularly nice family of Lagrangian systems called Finsler geometries, which directly generalize Riemannian geometries (classical mechanical systems) retaining their most alluring geometric properties, but also show that such systems are only one relative small generalization within a family we call {\em conservative fabrics}. 

There is a theoretical limitation to standard Lagrangian systems wherein both the priority metric and the policy that arise from them derive from the same source (the Lagrangian) and therefore cannot be designed separately. In terms of Finsler geometries, that means changing the policy will change the metric, and vice versa, making their direct use for behavior design challenging. We, therefore, derive a further generalization, a type of conservative fabric called a {\em geometric fabric}, which breaks that dependency by using Finsler geometries to define metrics, but then explicitly {\em bending} them to align with separately defined policy geometries while maintaining the metric. This class of geometric fabrics greatly generalizes the Riemannian geometries used in operational space control and give us a solid theoretical context within which to design provably stable behaviors with an expressivity on par with the broadest family of RMPs used in \cite{ratliff2018rmps} (where metrics and policies were designed separately and where both can be functions of both position {\em and} velocity enabling, for instance, barrier policies (such as obstacle or joint limit avoidance) to consider obstacles only when moving toward them).

The focus of this paper is on the theoretical foundations of optimization fabrics, so much of the empirical evaluation and comparison to earlier techniques such as operational space control is left to a separate work (see Sections VIII-IX in supplemental paper \cite{xiegeometricfabricsupp} currently under peer review). We do provide a simple demonstration of constructing a tabletop manipulation primitive on a 2D planar arm using geometric fabrics to show the feasibility and relative simplicity of using fabrics in practice. We exploit the geometric consistency of geometric fabrics by building the behavior in layers, tuning a part then freezing it before moving to the next component. For example, we first create reaching behavior, then add redundancy resolution behavior, and finally, add behavior that further shapes the paths of the end-effector. This strategy dramatically simplifies the tuning process making hand-tuning of these policies relatively straightforward. This approach is similar to how one might construct a hierarchical controller with hierarchical sequence of null space tasks, but here we use soft priorities rather than hard null spaces for more flexible expression of trade-offs and influence.

\subsection{Related work: An encompassing framework for existing techniques}
\label{sec:RelatedWork}

The literature on behavior design is diverse. Since accelerations are a natural action in mechanical systems where state constitutes position {\em and} velocity we focus exclusively on second-order models capable of capturing {\em fast reactive} behavior appropriate for real-world collaborative systems. The most prominent and standard technique in this class is operational space control \cite{KhatibOperationalSpaceControl1987,Peters_AR_2008}. \cite{2017_rss_system} showed that these reactive subsystems are critical in real-time systems even when planning \cite{LavallePlanningAlgorithms06} or optimal control \cite{RatliffCHOMP2009,ToussaintTrajOptICML2009,AbbeelTrajopt2013,RIEMORatliff2015ICRA} modules are available since there is always contention between modeling complexity and computational cost in these planners.


Operational space control is a subset of geometric control \cite{bullo2004geometric}, and both model behavior using classical mechanics \cite{ClassicalMechanicsTaylor05}.  Within the context of optimization fabrics, they are a specific (Riemannian) type of {\em Finsler} fabrics, with the restriction that the priority metric be a function of position {\em only}. Riemannian Motion Policies (RMPs) were developed to circumvent these limitations \cite{ratliff2018rmps}, but despite strong empirical success, they lacked stability guarantees and theoretical models of convergence. \cite{cheng2018rmpflow} developed Geometric Dynamical Systems (GDSs) within the RMP framework as a theoretically rigorous and stable model of RMPs. However, GDSs are non-covariant versions of {\em Lagrangian} fabrics (aside from covariance issues, which manifest in the definition of {\em structured} GDS, they have very similar theoretical properties), and are fundamentally limited in their modeling capacity as we describe below. We refer to operational space control, geometric control, and GDSs all as Lagrangian fabrics here. 

Lagrangian fabrics are limited by their {\em force-based design}. Behavior is shaped in two ways: by forcing the system with a potential and by shaping the underlying fabric's dynamics. Adding together multiple potentials, a common approach, will always result in conflicting priorities. For instance, joint limit or obstacle avoidance potentials will make it impossible to reach a target at times, even when it is physically achievable. On the other hand, shaping the dynamics affects the derived priority metric and makes combining multiple component Lagrangian fabrics challenging. 



Importantly, within the theory of fabrics, Lagrangian fabrics constitute only a small portion of optimization fabrics. 
More flexible fabrics exist, such as conservative and geometric fabrics. Many of these fabrics
{\em use} Lagrangian systems to derive their {\em priority metrics} but also provision separate design strategies for deriving their {\em policies}. Essentially, Lagrangian fabrics only represent half of the problem, and that limited capacity competes with itself when attempting to express both metrics {\em and} policies for modular design.
Moreover, earlier work deriving GDSs \cite{cheng2018rmpflow} focused on theoretical properties and provided few tools for the practical use of these systems. Designing behavior within these frameworks was, therefore, challenging, which inspired research into learning the systems from data \cite{rana2019LearningRmpsFromDemonstration,mukadam2019RMPFusion,li2019MultiAgentRMPsArXiv}. Those trained systems, however, suffered from the limitations discussed above and proved difficult to compose as modular components.

Another class of second-order systems common in robotics is Dynamic Movement Primitives (DMPs) \cite{ijspeert2013DMPs}. DMPs inherit their stability guarantees by reducing over time to linear systems as the nonlinear components vanish by design. That reduction provides stability, but again at the cost of expressivity. These tools have  seen substantial success in practice for basic motion primitive representation, especially in the context of imitation and reinforcement learning \cite{ijspeert2013DMPs,pastorLearningMotorSkillsByDemo2009,theodorouPathIntegralRL2010}, but techniques such as \cite{parkPastor2008AdaptingDMPs} for moving beyond ego-motion to obstacle avoidance follow more ad hoc heuristic coupling term methodologies similar to related potential field ideas \cite{KhatibPotentialFields1985,haddadin2010iros}. These techniques are largely subsumed by geometric control in terms of flexibility and stability. 
\fullversiononly

\section{Preliminaries} \label{sec:preliminaries}

\subsection{Manifolds and notation}
Following \cite{finslerGeometryForRoboticsArXiv2020} we base our notation on advanced calculus rather than standard tensor-based presentations. See that work for details, including details of the calculus notation used.
We restrict attention to smooth manifolds \cite{LeeSmoothManifolds2012} and assume that vector fields, functions, trajectories, etc. defined on the manifolds are smooth. An important concept is that of a manifold with a boundary, which we define next.

\begin{definition}
Let $\mathcal{X}$ be a manifold with a boundary (see \cite{LeeSmoothManifolds2012} for a definition\footnote{Intuitively, it can be thought of as a set theoretic boundary point in coordinates.}). Its boundary is denoted $\partial\mathcal{X}$ and we denote its {\em interior} as $\mathrm{int}(\mathcal{X}) = \mathcal{X}\backslash\partial\mathcal{X}$. 
\end{definition}

Any manifold $\mathcal{X}$ notationally has a boundary $\partial\mathcal{X}$, but that boundary could be empty $\partial\mathcal{X} = \emptyset$.

\begin{definition}
Let $\mathcal{X}$ denote a smooth manifold of dimension $n$. The space of all velocities at $\x$ is the {\em tangent space} at $\x$ denoted $\mathcal{T}_\x\mathcal{X}$. The set of all positions and velocities the {\em tangent bundle} denoted $\mathcal{T}\mathcal{X} = \coprod_{\x\in\mathcal{X}}\mathcal{T}_\x\mathcal{X}$\footnote{The symbol $\coprod$ denotes the disjoint union. I.e. $(\x,\xd)\in \mathcal{T}\mathcal{X}$ if and only if $\xd\in\mathcal{T}_\x\mathcal{X}$ for some $\x\in\mathcal{X}$.} 
The boundary of a manifold $\partial\mathcal{X}$ is a smooth manifold of dimension $n-1$, with $n-1$ dimensional tangent bundle $\mathcal{T}\partial\mathcal{X} = \coprod_{\x\in\partial\mathcal{X}}\mathcal{T}_\x\partial\mathcal{X}$. Likewise, $\mathrm{int}(\mathcal{X})$ is the manifold of dimension $n$ consisting of all interior points; its tangent bundle is denoted $\mathcal{T}\mathrm{int}(\mathcal{X}) = \coprod_{\x\in\mathrm{int}(\mathcal{X})}\mathcal{T}_\x\mathcal{X}$. The complete manifold with a boundary is understood to be the disjoint union of the interior and boundary manifolds $\mathcal{X}=\mathrm{int}(\mathcal{X})\coprod\partial\mathcal{X}$, and its tangent bundle is the disjoint union of the separate tangent bundles $\mathcal{T}\mathcal{X} = \mathcal{T}\mathrm{int}(\mathcal{X})\coprod\mathcal{T}\partial\mathcal{X}$.
\end{definition}
\begin{remark}
The notation $(\x,\xd)\in\mathcal{T}\mathcal{X}$ is understood to denote t0wo separate cases $(\x,\xd)\in\mathcal{T}\mathrm{int}(\mathcal{X})$ {\em and} $(\x,\xd)\in\mathcal{T}\partial\mathcal{X}$. In particular, that second case restricts the velocities to lie tangent to the boundary's surface.
\end{remark}

\fi  
\compressedversiononly  
\section{Spectral semi-sprays (specs)}
\else 
\section{Spectral semi-sprays (specs) and transform trees}
\fi 
\label{sec:SpecsAndTransformTrees}

We use a generalized notion of a Riemannian Motion Policy (RMP) \cite{ratliff2018rmps,cheng2018rmpflow} as a building block for this theory, a mathematical object we call a spectral semi-spray (spec).\footnote{A semi-spray in mathematics is the differential equation $\xdd + \h(\x, \xd) = \zero$ with $\h = \M^{-1}\f$ \cite{shenSprayGeometry2000}. The spectral (metric) matrix $\M$ plays an important role in our algebra, prompting our definition of {\em spectral} semi-sprays.} Second-order differential equations of the form $\M(\x, \xd)\xdd + \f(\x, \xd) = \zero$ generally follow an algebra of summation and pullback across maps $\x = \phi(\q)$ (composition with $\xd = \J\qd$ and $\xdd = \J\qdd + \Jd\qd$) that parallels RMPs. Specs have a {\em natural} $(\M, \f)_\mathcal{X}$ and {\em canonical} $(\M, \M^{-1}\f)_\mathcal{X}^\mathcal{C}$ form as in \cite{cheng2018rmpflow}. Rather than attributing semantics of policies from the outset, we use these objects more generally (for instance, to track energies across a transform tree) and add policy semantics as a specific policy-form spec $[\M, \pi]_\mathcal{X}$ with $\pi = -\M^{-1}\f$ (distinguishing this separate form using square brackets $[\cdot,\cdot]$) so that the corresponding differential equation is $\xdd = \pi(\x, \xd)$.
\compressedversiononly
\fi

\fullversiononly

\subsection{Pullback and combination of specs}
\label{sec:SpecAlgebra}

Given a differentiable map $\phi:\mathcal{Q}\rightarrow\mathcal{X}$ with action denoted $\x = \phi(\q)$, we can derive an expression for the covariant transformation of a spec $(\M, \f)_\mathcal{X}$ on the codomain $\mathcal{X}$ to a spec $(\widetilde{\M}, \widetilde{\f})_\mathcal{Q}$ on the domain $\mathcal{Q}$. Noting $\xdd = \J\qdd + \Jd\qd$ with $\J = \partial_\q\phi$, the covariant transformation of left hand side of $\M\xdd + \f = \zero$ is
\begin{align*}
    &\J^\tr\big(\M\xdd + \f\big) = \J^\tr\Big(\M\big(\J\qdd + \Jd \qd\big) + \f\Big)\\
    &\ \ \ = \big(\J^\tr \M \J\big) \qdd + \J^\tr\big(\f + \M \Jd\qd\big) \\
    &\ \ \ = \widetilde{\M}\qdd + \widetilde{\f}.
\end{align*}
where $\widetilde{\M} = \J^\tr \M \J$ and $\widetilde{\f} = \J^\tr\big(\f + \M \Jd\qd\big)$. In terms of specs, this transform becomes what we call the pullback operation:
\begin{align*}
    \mathrm{pull}_\phi (\M,\f)_\mathcal{X} 
    = \left(\J^\tr\M\J,\ \J^\tr\big(\f + \M \Jd\qd\big)\right)_\mathcal{Q}.
\end{align*}
Similarly, since $\big(\M_1\xdd + \f_1\big) + \big(\M_2\xdd + \f_2\big) = \big(\M_1 + \M_2\big) + \big(\f_1 + \f_2\big)$, we can define an associative and commutative summation operation
\begin{align*}
    (\M_1,\f_1)_\mathcal{X} + (\M_2,\f_2)_\mathcal{X}
    = \big(\M_1 + \M_2,\ \f_1 + \f_2\big)_\mathcal{X}.
\end{align*}
The above operations are defined on {\em natural form} specs. 
Canonical form specs $(\M, \ma)_\mathcal{X}^\mathcal{C}$ represent an acceleration form of the equation $\xdd + \M^{-1}\f = \zero$, with $\ma = \M^{-1}\f$. In canonical form, the summation operation can be interpreted as a metric weighted average of accelerations:
\begin{align*}
    &(\M_1,\ma_1)^\mathcal{C}_\mathcal{X}
        + (\M_2,\ma_2)^\mathcal{C}_\mathcal{X}\\
    &= \Big(\M_1 + \M_2,
        (\M_1 + \M_2)^{-1}\big(\M_1\ma_1 + \M_2\ma_2\big)
      \Big)^\mathcal{C}_\mathcal{X}.
\end{align*}
(These expressions can be derived by converting to natural form and using the above definitions.) The operations of pullback and summation constitute what we call the \textit{spec algebra}, which allows us to populate transform trees with specs for efficient computation as was done in \cite{cheng2018rmpflow}.

\subsection{Specs on transform trees}
\label{sec:SpecsOnTransformTrees}

Specs are similar to Riemannian Motion Policies (RMPs) \cite{cheng2018rmpflow}, but more general. RMPs are semantically assumed to constitute pieces of a motion policy. Specs, on the other hand, can take a broader role wherever these differential equations and their associated algebra appear. For instance, in the case of geometric fabrics discussed below, they will appear as both energy equations and parts of a policy geometry. Both classes of specs then work together to define the full policy.

Similar to RMPs, specs can populate a transform tree of spaces. For instance, if $\phi_i:\mathcal{Q}\rightarrow\mathcal{X}_i$, $i=1,\ldots,n$ defines a star-shaped transform tree \cite{cheng2020RmpflowJournalArxiv} and specs $\{(\M_i,\f_i)\}_{i=1}^n$ are placed on the spaces, the accumulated spec at the root $\mathcal{Q}$ after pullback and combination is
\begin{align*}
    &\sum_{i=1}^n \mathrm{pull}_{\phi_i} (\M_i,\f_i)_{\mathcal{X}_i}\\
    &\ \ = \left(\sum_i\J_i^\tr\M_i\J_i,\ \sum_i\J_i^\tr\big(\f_i + \Jd_i\qd\big)\right)_\mathcal{Q}.
\end{align*}
Again, in canonical form, the final space may be viewed as a metric weighted average of individual pulled back specs.

\begin{definition}[Closure under tree operations]
A class of specs is said to be {\em closed under tree operations} (or {\em closed} for short when the context is clear) when specs resulting from the operations remain in the same class.
\end{definition}
Importantly, if a class of specs is closed, the accumulated spec at the root of a transform tree remains of the same class.

\fi 
\section{Optimization fabrics} \label{sec:OptimizationFabrics}

A spec can be \textit{forced} by a position dependent potential function $\psi(\x)$ using
\begin{align}
    \M\xdd + \f = -\partial_\x\psi,
\end{align}
where the gradient $-\partial_\x\psi$ defines the force added to the system. 
In most cases, forcing an arbitrary spec does not result in a system that's guaranteed to converge to a local minimum of $\psi$. But when it does, we say that the spec is \textit{optimizing} and forms an \textit{optimization fabric} or \textit{fabric} for short. Herein, we characterize optimization fabrics  of increasing specificity. We will use these results in Section~\ref{sec:GeometricFabrics} to derive what we call {\em geometric fabrics} which define a concrete set of tools for fabric design.

Note that the accelerations of a forced system $\xdd = -\M^{-1}\f - \M^{-1}\partial_\x\psi$ decompose into nominal accelerations of the system $-\M^{-1}\f$ and forced accelerations $-\M^{-1}\partial_\x\psi$; we can, therefore, interpret the fabric as defining a nominal behavior that the potential function then pushes away from during optimization.

\subsection{Unbiased specs and general optimization fabrics} \label{sec:UnbiasedGeneralFabrics}

\begin{definition}[Interior]
Let $\mathcal{X}$ be a manifold with boundary $\partial\mathcal{X}$ (possibly empty). A spec $(\M,\f)_\mathcal{X}$ is said to be {\em interior} if for any interior starting point 
$(\x_0, \xd_0) \in \mathcal{T}\mathrm{int}(\mathcal{X})$ the integral curve $\x(t)$ 
is everywhere interior $\x(t)\in\mathrm{int}(\mathcal{X})$ for all $t\geq0$.
\end{definition}

\begin{definition}[Rough and frictionless specs]
Let $\mathcal{X}$ be a manifold with a (possibly empty) boundary. A spec $\mathcal{S} = \big(\M,\f\big)_\mathcal{X}$ is said to be {\em rough} if all its integral curves $\x(t)$ converge: $\lim_{t\rightarrow\infty}\x(t) = \x_\infty$ with $\x_\infty\in \mathcal{X}$ (including the possibility $\x_\infty\in\partial\mathcal{X}$). If $\mathcal{S}$ is not rough, but each of its {\em damped variants} $\mathcal{S}_\B = \big(\M,\f+\B\xd\big)$ is, where $\B(\x,\xd)$ is smooth and positive definite, the spec is said to be {\em frictionless}. A frictionless spec's damped variants are also known as {\em rough variants} of the spec.
\end{definition}

\begin{definition}[Forcing]
Let $\psi(\x)$ be a smooth potential function with gradient $\partial_\x\psi$ and let $(\M,\f)_\mathcal{X}$ be a spec. Then $(\M,\f+\partial_\x\psi)$ is the spec's {\em forced variant} and we say that we are {\em forcing} the spec with potential $\psi$. We say that $\psi$ is {\em finite} if $\|\partial_\x\psi\|<\infty$ everywhere on $\mathcal{X}$.
\end{definition}

\begin{definition}[Fabrics]
A spec $\mathcal{S}$ forms a {\em rough fabric} if it is rough when forced by a finite potential $\psi(\x)$ and each convergent point $\x_\infty$ is a Karush–Kuhn–Tucker (KKT) \cite{NocedalWright2006} solution of the constrained optimization problem $\min_{\x\in\mathcal{X}}\psi(\x)$. A forced spec is a {\em frictionless fabric} if its rough variants form rough fabrics.
\end{definition}

\fullversiononly
\begin{lemma}
If a spec $\mathcal{S}$ forms a rough (or frictionless) fabric then (unforced) it is a rough (or frictionless) spec.
\end{lemma}
\begin{proof}
The zero function $\psi(\x) = 0$ is a finite potential under which all points in $\mathcal{X}$ are KKT solutions. If $\mathcal{S}$ is a rough fabric, then all integral curves of $\mathcal{S}$ are convergent and it must be a rough spec. 
\end{proof}
\fi 

\begin{definition}[Boundary conformance]\label{def:BoundaryConformingSpec}
A spec $\mathcal{S} = \big(\M,\f\big)_\mathcal{X}$ is {\em boundary conforming} if the following hold: 
\begin{enumerate}
\item $\mathcal{S}$ is interior.
\item $\M(\x,\vv)$ and $\f(\x,\vv)$ are finite for all\footnote{In this definition, we emphasize in the notation that the manifold contains both its interior and the boundary. In practice, proofs often treat both cases separately.} $(\x,\vv)\in\mathcal{T}\mathcal{X} = \mathcal{T}\mathrm{int}(\mathcal{X}) \cup \mathcal{T}\partial\mathcal{X}$.
\item For every tangent bundle trajectory $(\x(t), \vv(t)) \in \mathcal{T}\mathcal{X}$ with $\x\rightarrow\x_\infty\in\partial\mathcal{X}$, we have
$\lim_{t\rightarrow\infty}\left\|\M^{-1}(\x, \vv)\right\|<\infty$
and
$\lim_{t\rightarrow\infty}\big\|\V_\infty^\tr\f(\x,\vv)\big\| < \infty$  
where $\V_\infty$ is a matrix for which $\mathrm{span}(\V_\infty) = \mathcal{T}_{\x_\infty}\mathcal{X}$.
\end{enumerate}
A {\em boundary conforming metric} is a metric satisfying conditions (2) and (3) of this definition. Additionally, $\f$ is said to be {\em boundary conforming with respect to $\M$} if $\M$ is a boundary conforming metric and $\big(\M, \f\big)_\mathcal{X}$ forms a boundary conforming spec. 
\end{definition}
Note that this definition of boundary conformance implies that $\M$ either approaches a finite matrix along trajectories limiting to the boundary or it approaches a matrix that is finite along Eigen-directions parallel with the boundary's tangent space but explodes to infinity along the direction orthogonal to the tangent space. This means that either $\M^{-1}_\infty$ is full rank or it is reduced rank and its column space spans the boundary's tangent space $\mathcal{T}_{\x_\infty}\partial\mathcal{X}$.

\begin{definition}[Unbiased] \label{def:Unbiased}
A boundary conforming spec $\big(\M,\f\big)_\mathcal{X}$ is {\em unbiased} if for every convergent trajectory $\x(t)$ with $\x\rightarrow\x_\infty$ we have
$\V_\infty^\tr\f(\x,\xd) \rightarrow \zero$ where $\V_\infty$ is a matrix for which $\mathrm{span}(\V_\infty) = \mathcal{T}_{\x_\infty}\mathcal{X}$.
If the spec is not unbiased, we say it is {\em biased}. We often refer to the term $\f$ alone as either biased or unbiased when the context of $\M$ is clear. Note that unbiased implies boundary conforming.
\end{definition}
\begin{remark}
In the above definition, when $\x_\infty\in\mathrm{int}(\mathcal{X})$, the basis $\V_\infty$ contains a full set of $n$ linearly independent vectors, so the condition $\V_\infty^\tr\f(\x, \xd)\rightarrow\zero$ implies $\f(\x, \xd)\rightarrow\zero$. This is not the case for $\x_\infty\in\partial\mathcal{X}$. 
\end{remark}
\begin{remark}
The matrix $\V_\infty$ naturally defines two linearly independent subspaces, the column space, and the left null space. It is often convenient to decompose a vector into two components $\f = \f^\paral + \f^\perp$, one lying in the column space $\f^\paral$ and the other lying in the left null space $\f^\perp$. A number of the statements below are phrased in terms of such a decomposition.
\end{remark}
Since the definition of unbiased is predicated on the spec being boundary conforming, the spectrum of the metric $\M$ is always finite in the relevant directions (all directions for interior points and directions parallel to the boundary for boundary points). The property of being unbiased is, therefore, linked to zero acceleration within the relevant subspaces. This property is used in the following theorem to characterize general fabrics as unbiased optimization mediums. 

\begin{theorem}[General fabrics] \label{thm:GeneralFabrics}
Suppose $\mathcal{S} = \big(\M,\f\big)_\mathcal{X}$ is a 
boundary conforming spec. Then $\mathcal{S}$ forms a rough 
fabric if and only if it is unbiased and 
converges when forced by a potential $\psi(\x)$ with $\|\partial\psi\|<\infty$ on $\x\in\mathcal{X}$. 
\end{theorem}
\begin{proof}
The forced spec defines the equation
\begin{align}\label{eqn:ForcedGeneralSystem}
    \M\xdd + \f = -\partial_\x\psi.
\end{align}

We will first assume $\f$ is unbiased. Since $\x$ converges, we must have $\xd\rightarrow\zero$ which means $\xdd\rightarrow\zero$ as well. 

If $\x(t)$ converges to an interior point $\x_\infty\in\mathrm{int}(\mathcal{X})$, $\M$ is finite so $\M\xdd\rightarrow\zero$ since $\xdd\rightarrow\zero$. And since the spec is unbiased we have $\f(\x,\xd)\rightarrow\zero$ as $\xd\rightarrow\zero$. Therefore, the left hand side of Equation~\ref{eqn:ForcedGeneralSystem} approaches $\zero$, so $\partial_\x\psi\rightarrow\zero$ satisfying (unconstrained) KKT conditions. 

Alternatively, if $\x(t)$ converges to a boundary point $\x_\infty\in\partial\mathcal{X}$,
then we can analyze the expression
\begin{align} \label{eqn:BoundaryLimit}
    \xdd = -\M^{-1}\big(\f + \partial_\x\psi\big)\rightarrow\zero
\end{align}
since $\xdd\rightarrow\zero$.
Since $\M$ is boundary conforming, the inverse metric limit $\M_\infty^{-1}$ is finite, and since $\partial\psi$ is also finite on $\mathcal{X}$, the term $\M^{-1}\partial\psi$ converges to the finite vector $\M^{-1}_\infty\partial\psi_\infty$. Therefore, by Equation~\ref{eqn:BoundaryLimit} we also have $\M^{-1}\f \rightarrow \M^{-1}_\infty\f_\infty = -\M^{-1}_\infty\partial\psi_\infty$. At the limit, $\M^{-1}_\infty$ has full rank across $\mathcal{T}_{\x_\infty}\partial\mathcal{X}$, so the above limit equality implies $\f_\infty^\paral = -\partial\psi_\infty^\paral$, were $\f_\infty^\paral$ and $\partial\psi_\infty^\paral$ are the components of $\f_\infty$ and $\partial\psi_\infty$, respectively, lying in the boundary's tangent space $\mathcal{T}_{\x_\infty}\partial\mathcal{X}$. Since $\f$ is unbiased and $\x_\infty\in\partial\mathcal{X}$, the boundary parallel component $\f_\infty^\paral = \zero$ so it must be that $\partial\psi_\infty^\paral = \zero$ as well. Therefore, $\partial\psi_\infty$ must either be orthogonal to $\mathcal{T}_{\x_\infty}\partial\mathcal{X}$ or zero. If it is zero, the KKT conditions are automatically satisfied. If it is nonzero, since $\x(t)$ is interior, $-\f(\x,\xd)$ must be interior near the boundary, so $-\partial\psi_\infty$ must be exterior facing and balancing $-\f_\infty$. That orientation in addition to the orthogonality implies that the limiting point satisfies the (constrained) KKT conditions.

Finally, to prove the converse, assume $\f$ is biased. Then there exists a point $\x^*\in\mathcal{X}$ for which $\f(\x^*,\zero)\neq \zero$. We can easily construct an objective potential with a unique global minimum at $\x^*$. The forced system, in this case, cannot come to rest at $\x^*$ since $\f$ is nonzero there, so $(\M,\f)$ is not guaranteed to optimize and is not a fabric.
\end{proof}

The above theorem characterizes the most general class of fabric and shows that all fabrics are necessarily unbiased in the sense of Definition~\ref{def:Unbiased}. The theorem relies on hypothesizing that the system always converges when forced, and is therefore more of a template for proving a given system forms a fabric rather than a direct characterization. Proving convergence is in general nontrivial. The specific fabrics we introduce below will prove convergence using energy conservation and boundedness properties.

Note that the above theorem does not place restrictions on whether or not in the limit the metric is finite in Eigen-directions orthogonal to the boundary surface's tangent space. In practice, it can be convenient to allow metrics to raise to infinity in those directions, so the effects of forces orthogonal to the boundary's surface are increasingly reduced by the increasingly large mass. Such metrics can induce smoother optimization behavior when optimizing toward local minima on a boundary surface.

\subsection{Conservative fabrics} \label{sec:ConservativeFabrics}

Conservative fabrics, frictionless fabrics which conserve a well-defined form of energy, constitute a large class of practical fabrics. Here we define a class of energies through the Lagrangian formalism and extend the notion of boundary conformance to these energy Lagrangians. We give some essential lemmas on energy conservation and present the basic result on conservative fabrics in Proposition~\ref{prop:ConservativeFabrics}. Unlike the most general theorem on fabrics given above in Theorem~\ref{thm:GeneralFabrics}, here we no longer need to hypothesize that the system converges. The conservation properties enable us to prove convergence by effectively using the energy as a Lyapunov function \cite{khalil2002nonlinear}.  

Two standard classes of conservative fabrics that naturally arise are Lagrangian and Finsler fabrics (the latter a subclass of the former), both of which are defined by the equations of motion of their respective classes of energy Lagrangians. The energy-conservation properties of Lagrangian systems are well-understood from classical mechanics \cite{ClassicalMechanicsTaylor05}; this observation simply links those result to our fabric framework. In subsection~\ref{sec:EnergizationAndEnergizedFabrics}, though, we show how to create a much broader class of conservative system by applying an {\em energization} transform to a differential equation resulting in a new type of conservative fabric known as an energized fabric.

\begin{definition}[Lagrangian spec]
A stationary Lagrangian $\Lag(\x,\xd)$ is {\em boundary conforming (on $\mathcal{X}$)} if its induced equations of motion $\partial^2_{\xd\xd}\Lag\,\xdd + \partial_{\xd\x}\Lag\,\xd - \partial_\x\Lag = \zero$ under the Euler-Lagrange equation form a boundary conforming spec $\mathcal{S}_\Lag = \big(\M_\Lag, \f_\Lag\big)_\mathcal{X}$ where $\M_\Lag = \partial^2_{\xd\xd}\Lag$ and $\f_\Lag = \partial_{\xd\x}\Lag\,\xd - \partial_\x\Lag$. This spec is known as the {\em Lagrangian spec} associated with $\Lag$. $\Lag$ is additionally {\em unbiased} if $\mathcal{S}_\Lag$ is unbiased. Since unbiased specs are boundary conforming by definition, an unbiased Lagrangian $\Lag$ is implicitly boundary conforming as well.
\end{definition}

\begin{definition}[Energy]
Let $\Lag_e(\x, \xd)$ be a stationary Lagrangian with Hamiltonian $\Ham_e(\x, \xd) = \partial_\xd\Lag_e^\tr\,\xd - \Lag_e$. $\Lag_e$ is an {\em energy Lagrangian} if $\M_e = \partial^2_{\xd\xd}\Lag_e$ is full rank, $\Ham_e$ is nontrivial (not everywhere zero), and $\Ham_e(\x, \xd)$ is finite on $\mathrm{int}(\mathcal{X})$.
An energy Lagrangian's equations of motion $\M_e\xdd + \f_e = \zero$ are often referred to as its {\em energy equations} with spec denoted $\mathcal{S}_e = \big(\M_e, \f_e\big)_\mathcal{X}$.
\end{definition}

\fullversiononly
\begin{definition}[Boundary intersecting trajectory]
Let $\x(t)$ be a trajectory. If there exists a $t_0 < \infty$ such $\x(t_0)\in\partial\mathcal{X}$ and $\xd\notin\mathcal{T}_{\x(t_0)}\partial\mathcal{X}$ it is said to be {\em boundary intersecting} with {\em intersection time} $t_0$.
\end{definition}

\begin{definition}[Energy boundary limiting condition]
Let $\Lag_e$ be an energy Lagrangian with energy $\Ham_e$. If for every boundary intersecting trajectory $\x(t)$ with intersection time $t_0$ we have $\lim_{t\rightarrow t_0} \Ham_e(\x, \xd) = \infty$, we say that $\Ham_e$ satisfies the {\em boundary limiting condition}.
\end{definition}

\begin{remark}
When designing boundary conforming energy Lagrangians one must ensure the Lagrangian's spec is interior. To attain an interior spec, it is often helpful to design energies that prevent energy conserving trajectories from intersecting $\partial\mathcal{X}$. It can be shown that an energy Lagrangian's spec is interior if and only if its energy is boundary limiting, i.e. the energy approaches infinity for any boundary intersecting trajectory.
\end{remark}
\fi 

\begin{lemma} \label{lma:EnergyConservation}
    Let $\Lag_e$ be an energy Lagrangian. Then $\M_e\xdd + \f_e + \f_f = \zero$ is energy conserving if and only if $\xd^\tr\f_f = \zero$. We call such a term a {\em zero work modification}. Such a spec $\mathcal{S} = \big(\M_e, \f_e + \f_f\big)$ is said to be a {\em conservative} spec under energy Lagrangian $\Lag_e$.
\end{lemma}
\begin{proof}
This energy is conserved if its time derivative is zero. 
Substituting $\xdd = -\M_e^{-1}\big(\f_e + \f_f\big)$ into Equation~\ref{eqn:EnergyTimeDerivative} of Lemma~\ref{lma:EnergyTimeDerivative} 
and setting it to zero gives
\begin{align*}
    \dot{\Ham}_{\Lag_e} 
    &= \xd^\tr\Big(\M_e \big(-\M_e^{-1}(\f_e + \f_f)\big)+ \f_e\Big) \\
    &= \xd^\tr\big(-\f_e - \f_f + \f_e\big) \\
    &= \xd^\tr\f_f = 0.
\end{align*}
Therefore, energy is conserved if and only if final constraint holds.
\end{proof}

\begin{proposition}[Conservative fabrics] \label{prop:ConservativeFabrics}
Suppose $\mathcal{S} = \big(\M_e, \f_e + \f_f\big)_\mathcal{X}$ is a conservative unbiased spec under energy Lagrangian $\Lag_e$ with zero work term $\f_f$. Then $\mathcal{S}$ forms a frictionless fabric.
\end{proposition}
\begin{proof}
Let $\psi(\x)$ be a lower bounded finite potential function. 
\compressedversiononly
Using $\dot{\Ham}_e = \xd^\tr\big(\M_e\xdd + \f_e\big)$, we 
\else 
Using Lemma~\ref{lma:EnergyTimeDerivative} we 
\fi 
can derive an expression for how the total energy $\Ham_e^{\psi} = \Ham_e + \psi(\x)$ varies over time:
\begin{align}
    \dot{H}_e^\psi 
    &= \dot{H}_e + \dot{\psi}
    = \xd^\tr\big(\M_e\xdd + \f_e\big) + \partial_\x\psi^\tr \xd \\
    \label{eqn:TotalEnergyTimeDerivative}
    &= \xd^\tr\big(\M_e\xdd + \f_e + \partial_\x\psi\big).
\end{align}
With damping matrix $\B(\x, \xd)$ let $\M_e\xdd + \f_e + \f_f = -\partial_\x\psi - \B\xd$ be the forced and damped variant of the conservative spec's system. Subtituting into Equation~\ref{eqn:TotalEnergyTimeDerivative} gives
\begin{align} \label{eqn:SystemEnergyDecrease}
    \nonumber
    \dot{H}_e^\psi
    &= \xd^\tr\Big(\M_e\big(-\M_e^{-1}(\f_e + \f_f + \partial_\x\psi + \B\xd)\big) \\\nonumber
    &\ \ \ \ \ \ \ \ \ \ \ \ + \f_e + \partial_\x\psi\Big)\\
    \nonumber
    &= \xd^\tr\Big(-\f_e - \partial_\x\psi - \B\xd + \f_e + \partial_\x\psi\Big) - \xd^\tr\f_f \\
    &= -\xd^\tr\B\xd
\end{align}
since all terms cancel except for the damping term. When $\B$ is strictly positive definite, the rate of change is strictly negative for $\xd\neq\zero$. Since $\Ham_e^{\psi} = \Ham_e + \psi$ is lower bounded and $\dot{\Ham}_e^{\psi}\leq\zero$, we must have $\dot{\Ham}_e^\psi = -\xd^\tr\B\xd \rightarrow \zero$ which implies $\xd\rightarrow\zero$, and therefore, system convergence.  
Since the system is additionally boundary conforming and unbiased, by Theorem~\ref{thm:GeneralFabrics} it forms a fabric.

Finally, when $\B = \zero$ Equation~\ref{eqn:SystemEnergyDecrease} shows that total energy is conserved, so a system that starts with nonzero energy cannot converge. Therefore, the undamped system is a frictionless fabrics with rough variants defined by the added damping term.
\end{proof}

\begin{corollary}[Lagrangian and Finsler fabrics]
If $\Lag_e(\x, \xd)$ is an unbiased energy Lagrangian, then $\mathcal{S}_e = \big(\M_e, \f_e\big)_\mathcal{X}$ forms a frictionless fabric known as a {\em Lagrangian fabric}. When $\Lag_e$ is a Finsler energy (see Definition~\ref{def:FinslerStructure}), this fabric is known more specifically as a {\em Finsler fabric}.
\end{corollary}
\begin{proof}
$\M_e \xdd + \f_e = \zero$ is conservative by Lemma~\ref{lma:EnergyConservation} with $\f_f = \zero$ and unbiased, by Proposition~\ref{prop:ConservativeFabrics} it forms a frictionless fabric. 
\end{proof}

\subsection{Energization and energized fabrics}
\label{sec:EnergizationAndEnergizedFabrics}

\begin{corollary}
Let $\Lag_e$ be an energy Lagrangian and let $\wt{\f}_f(\x, \xd)$ be any forcing term. Then using the projected term $\f_f = \mP_e\wt{\f}_f$ the forced equations of motion
\begin{align} \label{eqn:ForcedEnergySystem}
    \M_e\xdd + \f_e + \f_f = \zero 
\end{align}
are energy conserving.
\end{corollary}

\begin{proof}
By Lemma~\ref{lma:EnergyProjection}, $\xd^\tr\f_f = \xd^\tr\mP_e\f = \zero$, so by Lemma~\ref{lma:EnergyConservation}, Equation~\ref{eqn:ForcedEnergySystem} is energy conserving. 
\end{proof}

The following result characterizes what we call the {\em energization transform}. Energization accelerates a system along the direction of motion to conserve a given measure of energy.

\begin{proposition}[System energization] \label{prop:SystemEnergization}
Let $\xdd + \h(\x, \xd) = \zero$ be a differential equation, and suppose $\Lag_e$ is any energy Lagrangian with equations of motion $\M_e\xdd + \f_e = \zero$ and energy $\Ham_e$. Then $\xdd + \h(\x, \xd) + \alpha_{\Ham_e}\xd = \zero$ is energy conserving when
\begin{align}\label{eqn:EnergizationTransformAlpha}
    \alpha_{\Ham_e} = -(\xd^\tr\M_e\xd)^{-1}\xd^\tr\big[\M_e\h - \f_e\big],
\end{align}
and differs from the original system by only an acceleration along the direction of motion. 
The new system can be expressed as:
\begin{align} \label{eqn:ZeroWorkEnergizationForm}
  \M_e\xdd + \f_e + \mP_e\big[\M_e\h - \f_e\big] = \zero.
\end{align}
This modified system is known as an {\em energized system}. 
\fullversiononly
Moreover, the operation of {\em energizing} the original system is known as an {\em energization transform} and is denoted using spec notation as
\begin{align}
    \mathcal{S}_{\h}^{\Lag_e} 
    &= \Big(\M_e, \f_e + \mP_e\big[\M_e\h - \f_e\big]\Big)_\mathcal{X} 
    \\\nonumber
    &= \mathrm{energize}_{\Lag_e}\Big\{\mathcal{S}_\h\Big\},
\end{align}
where $\mathcal{S}_\h = \big(\I, \h\big)_\mathcal{X}$ is the spec representation of $\xdd + \h(\x,\xd) = \zero$.
\fi 
\end{proposition}
\begin{proof}
See the supplementary appendix.
\end{proof}

\begin{definition}
A metric $\M(\x, \xd)$ is said to be {\em boundary aligned} if for any convergent $\x(t)$ with $\x_\infty\in\partial\mathcal{X}$ the limit $\lim_{t\rightarrow\infty}\M^{-1}(\x,\xd) = \M_\infty^{-1}$ exists and is finite, and $\mathcal{T}_{\x_\infty}\partial\mathcal{X}$ is spanned by a subset of Eigen-basis of $\M_\infty^{-1}$.
\end{definition}

\begin{theorem}[Energized fabrics] \label{thm:EnergizedFabrics}
Let $\Lag_e$ be an unbiased energy Lagrangian with boundary aligned $\M_e = \partial^2_{\xd\xd}\Lag_e$ and lower bounded energy $\Ham_e$, and let $\big(\I, \h\big)$ be an unbiased spec. Then the energized spec $\mathcal{S}_\h^{\Lag_e} = \mathrm{energize}_{\Lag_e}\big\{\mathcal{S}_\h\big\}$ given by Proposition~\ref{prop:SystemEnergization} forms a frictionless fabric.
\end{theorem}
\begin{proof}
This result follows from 
Proposition~\ref{prop:ConservativeFabrics}, and Lemma~\ref{lma:EnergizationPreservesUnbiasedProperty} (see the appendix).
\end{proof}

\section{Geometric fabrics} \label{sec:GeometricFabrics}

A broad class of conservative fabrics is formed by energizing differential equations. Unfortunately, energization will usually change the system's behavior. In this section, we develop a class of {\em geometric} fabric which does not suffer from that problem---energization leaves the system's paths unchanged. See supplemental paper \cite{ratliffgeneralizednonlinearaccepted} for a detailed overview of the general nonlinear and Finsler geometries used here.

\subsection{Nonlinear geometries and Finsler energies} \label{sec:NonlinearGeometries}

\begin{definition}
A {\em generalized nonlinear geometry} is a geometry of paths characterized by its {\em generator}, a ordinary second-order differential equation of the form
\begin{align}
    \xdd + \h_2(\x,\xd) = \zero,
\end{align}
where $\h_2(\x,\xd)$ is a smooth, covariant, map $\h_2:\R^d\times\R^d\rightarrow\R^d$ that is {\em positively homogeneous of degree 2} in velocities in the sense $\h_2(\x, \alpha\xd) = \alpha^2\h_2(\x,\xd)$ for $\alpha>0$.

\fullversiononly
Solutions to the generator are called {\em generating} solutions, and if those trajectories are guaranteed to conserve an known energy quantity, they are known as {\em energy levels}.
\fi 
\end{definition}
\noindent Geometry paths are independent of speed.

Finsler geometry is the study of nonlinear geometries whose generating equation is defined by the equations of motion
of {\em Finsler energies}, which are defined in terms of the geometry's {\em Finsler structure}.
\begin{definition}[Finsler structure] \label{def:FinslerStructure}
A {\em Finsler structure} is a stationary Lagrangian $\Lag_g(\x,\xd)$ with the following properties:
\begin{enumerate}
    \item Positivity: $\Lag_g(\x,\xd) > 0 $ for all $\xd \neq \zero$.
    \item Homogeneity: $\Lag_g$ is {\em positively homogeneous of degree 1} in velocities in the sense $\Lag_g(\x,\alpha\xd) = \alpha \Lag_g(\x,\xd)$ for $\alpha>0$.
    \item Energy tensor invertibility: $\partial^2_{\xd\xd}\Lag_e$ is everywhere invertible, where $\Lag_e = \frac{1}{2}\Lag_g^2$.
\end{enumerate}
$\Lag_e$ is known as the {\em energy form} of $\Lag_g$, known as the {\em Finsler energy}.
\end{definition}
\fullversiononly
\begin{remark}
Note that the first two conditions together mean that $\Lag(\x,\zero) = 0$.
\end{remark}

Note that many texts replace that third condition (invertibility of the energy tensor) with a positive definiteness requirement. That positive definiteness is difficult to satisfy in practice in Finsler structure design, and the basic proofs in our construction detailed in \cite{finslerGeometryForRoboticsArXiv2020} require only invertibility, so we use the looser requirement in our definition.
\fi 



The following fundamental results on Finsler geometry are given without proof (see supplemental paper \cite{ratliffgeneralizednonlinearaccepted} for details.)
\begin{lemma}[Homogeneity of the Finsler energy tensor] \label{lma:FinslerEnergyHomogeneity}
Let $\Lag_g$ be a Finsler structure with energy form $\Lag_e = \frac{1}{2}\Lag_g^2$ and let $\M_e\xdd + \f_e = \zero$ be its equations of motion. Then $\M_e$ is homogeneous of degree 0 and $\f_e$ is homogeneous of degree 2.
\end{lemma}
The above lemma means that $\M_e$ is dependent on velocity $\xd$ only through its norm $\widehat{\xd}$, i.e. rescaling $\xd$ does not affect the energy tensor. 
The next lemma links that geometry generator to the Finsler structure's equations of motion as well.

\begin{theorem}[Finsler geometry generation] \label{thm:FinslerGeometries}
Let $\Lag_g$ be a Finsler structure with energy $\Lag_e = \frac{1}{2}\Lag_g^2$. The equations of motion of $\Lag_e$ define a geometry generator whose geometric equation is given by the equations of motion of $\Lag_g$.
\end{theorem}
This result tells us that Finsler structures $\Lag_g$ define nonlinear geometries of paths whose generating trajectories constitute energy levels of the energy $\Lag_e$ (since energy is conserved by the energy equations).

\subsection{Bent Finsler geometries form geometric fabrics} \label{sec:BentFinsler}

Equation~\ref{eqn:ZeroWorkEnergizationForm} shows that energization is a zero work modification to the energy equations. When the original equation is a geometry and the energy is Finsler, the energized equation generates a geometry equivalent to the original equation's geometry. That means we can view energization {\em bending} the {\em Finsler geometry} to match the desired geometry without affecting the system energy.

\begin{corollary}[Bent Finsler Representation] \label{cor:BentFinslerRepresentation}
Suppose $\h_2(\x, \xd)$ is homogeneous of degree 2 so that $\xdd + \h_2(\x, \xd) = \zero$ is a geometry generator, and let $\Lag_e$ be a Finsler structure (and therefore also a Finsler energy). Then the energized system $\M_e\xdd + \f_e + \mP_e\big[\M_e\h_2 - \f_e\big] = \zero$ is a geometry generator whose geometry matches the original system's geometry. Since the Finsler system $\M_e\xdd + \f_e = \zero$ is a geometry generator as well, we can view the energized system as a {\em zero work geometric modification} to the Finsler geometry, what we call a {\em bending} of the geometric system.
\end{corollary}
\begin{proof}
The energized system takes the form $\xdd + \wt{\h}_2(\x, \xd) = \zero$ where
\begin{align}
    \wt{\h}_2 = \M_e^{-1}\f_e + \mR_e\big[\M_e\h_2 - \f_e\big]
\end{align}
since $\mP_e = \M_e\mR_{\p_e}$.
The energy $\Lag_e$ is Finsler, so $\f_e$ is homogeneous of degree 2 and $\M_e$ is homogeneous of degree 0, which means the first term in combination is homogeneous of degree 2. Moreover, $\mR_{\p_e} = \M_e^{-1} - \frac{\xd\,\xd^\tr}{\xd^
\tr\M_e\xd}$ is homogeneous of degree 0 since the numerator and denominator scalars would cancel in the second term when $\xd$ is scaled. Therefore, the energized system in its entirety forms a geometry generator.

Since $\xdd + \h_2(\x, \xd) = \zero$ is a geometry generator, instantaneous accelerations along the direction of motion $\xd$ do not change the paths taken by the system. So $\xdd + \wt{\h}_2(\x, \xd) = \zero$ forms a generator whose geometry matches the original geometry defined by $\xdd + \h_2(\x, \xd) = \zero$. 
\end{proof}
\begin{remark}
Corollary~\ref{cor:BentFinslerRepresentation} shows that geometries are invariant under energization by Finsler energies.
\end{remark}

Note that for the energized system to be a generator, we need two properties: first, the original system must be a generator, and second, the energy must be Finsler. If the energy is not Finsler, the resulting energized system will still follow the same paths as the original geometry (since by definition it is formed by accelerating along the direction of motion), but it will not be itself a generator (the resulting differential equation will not produce path aligned trajectories when solved for differing initial speeds).

The following proposition is one of the key results that makes geometry generators useful for fabric design.
\begin{proposition} \label{prop:GeometryGeneratorsAreUnbiased}
Boundary conforming Finsler energies and geometry generators are unbiased.
\end{proposition}
\begin{proof}
Denote the generator by $\xdd + \h_2(\x, \xd) = \zero$. Let $\lambda(t) = \|\xd(t)\|$ so that $\xd(t) = \lambda(t)\,\|\widehat{\xd}(t)\|$. Since $\xd\rightarrow\zero$, $\lambda\rightarrow 0$. Moreover, since $\h_2$ is homogeneous of degree 2, $\h_2\big(\x(t), \xd(t)\big) = \lambda(t)\,\h_2\big(\x(t), \widehat{\xd}(t)\big)$. Therefore, 
\begin{align}
    &\big\|\V_\infty^\tr\h_2\big(\x(t), \xd(t)\big)\big\| \\\nonumber
    &\ \ \ \ \ \ = \lambda(t) \big\|\V_\infty^\tr \h_2\big(\x(t), \widehat{\xd}(t)\big)\big\|
    \rightarrow \zero
\end{align}
since $\big\|\V_\infty^\tr \h_2\big(\x(t), \widehat{\xd}(t)\big)\big\|$ is finite in the limit by definition of boundary conformance.

The equations of motion of a Finsler energy form a geometry generator by Theorem~\ref{thm:FinslerGeometries}. Since the energy is boundary conforming, so too is this geometry generator. Therefore, by Proposition~\ref{prop:GeometryGeneratorsAreUnbiased} the equations of motion are unbiased, so by definition the Finsler energy is unbiased. 
\end{proof}

\begin{corollary}[Geometric fabrics] \label{cor:GeometricFabrics}
Suppose $\h_2(\x, \xd)$ is homogeneous of degree 2 and unbiased so that $\xdd + \h_2(\x, \xd)  = \zero$ is an unbiased geometry generator, and suppose $\Lag_e$ is Finsler and boundary conforming. Then the energized system is fabric defined by a generator whose geometry matches the original generator's geometry. Such a fabric is called a {\em geometric fabric}.
\end{corollary}
\begin{proof}
By Corollary \ref{cor:BentFinslerRepresentation} the energized system is a generator with matching geometry, and by Theorem~\ref{thm:EnergizedFabrics} that energized system forms a fabric.
\end{proof}



\fullversiononly
\section{Transform trees and closure}

The following theorem shows that, it doesn't matter whether we first energize in the co-domain of a differentiable map and then pullback, or first pullback to the domain and then energize there. Both are equivalent. That means energization is a covariant operation; the resulting behavior is independent of coordinates. In practice, we design behaviors on transform trees and always pullback first to the root and energize there. That allows us to understand the combined behavior as a metric weighted average of individual policies, which are subsequently speed modulated for stability. Section~\ref{sec:speed_control} shows how extend these ideas to maintain stability while regulating an arbitrary {\em execution} energy to explicitly control a desired measure of speed.

\begin{theorem}[Energization commutes with pullback] \label{thm:EnergizationCommutesWithPullback}
Let $\Lag_e$ be an energy Lagrangian, and let $\xdd + \h(\x,\xd) = \zero$ be a second-order differential equation with associated natural form spec $(\M_e, \f)$ under metric $\M_e = \partial^2_{\xd\xd}\Lag_e$ where $\f = \M_e\h$. Suppose $\x = \phi(\q)$ is a differentiable map for which the pullback metric $\J^\tr\M_e\J$ is full rank. Then
\begin{align*}
    &\mathrm{energize}_{\mathrm{pull}\Lag_e}\Big(
        \mathrm{pull}_\phi \big(\M_e,\f_2\big)
    \Big)\\
    &\ \ \ \ \ \ \ \ \ \ \ \ =
    \mathrm{pull}_\phi\Big(
        \mathrm{energize}_{\Lag_e} \big(\M_e, \f_2\big)
    \Big).
\end{align*}
We say that the energization operation commutes with the pullback transform.
\end{theorem}
\begin{proof}
We will show the equivalence by calculation. The energization of $\xdd + \h = \zero$ in force form $\M_e\xdd + \f = \zero$ with $\f = \M_e\h$ is $\M_e\xdd + \f_e^{\h}$ where
\begin{align}
    \f_e^{\h} = \f_e + \M_e\left[\M_e^{-1} - \frac{\xd\xd^\tr}{\xd^\tr\M_e\xd}\right] \big(\f - \f_e\big),
\end{align}
where $\f_e = \partial_{\xd\x}\Lag_e\xd - \partial_\x\Lag_e$ so that $\M_e\xdd + \f_e = \zero$ is the energy equation. Let $\J = \partial_\x\phi$. The pullback of the energized geometry generator is
\begin{align}
    \nonumber
    &\J^\tr\M_e\left(\J\qdd + \Jd\qd\right) + \J^\tr\f_e^{\h} = \zero \\\nonumber
    &\Rightarrow \big(\J^\tr\M_e\J\big)\qdd + \J^\tr\big(\f_e^{\h} + \M_e\Jd\qd\big) = \zero\\\nonumber
    &\Rightarrow \big(\J^\tr\M_e\J\big)\qdd + \J^\tr \f_e \\\nonumber
    &\ \ \ \ \ \ \ \ + \J^\tr\M_e\left[\M_e^{-1} - \frac{\xd\xd^\tr}{\xd^\tr\M_e\xd}\right]
            \big(\f - \f_e\big) \\\nonumber
    &\ \ \ \ \ \ \ \ + \J^\tr\M_e\Jd\qd = \zero \\\label{eqn:EnergizationPullback}
    &\Rightarrow \wt{\M}_e\qdd + \wt{\f}_e \\\nonumber
    &\ \ \ \ \ \ \ \ + \J^\tr\M_e\left[\M_e^{-1} - \frac{\xd\xd^\tr}{\xd^\tr\M_e\xd}\right]
            \big(\f - \f_e\big),
\end{align}
where $\wt{\M}_e = \J^\tr\M_e\J$ and $\wt{\f}_e = \J^\tr\big(\f_e + \M_e\Jd\qd\big)$ form the standard pullback of $(\M_e, \f_e)$.

We can calculate the geometry pullback with respect to the energy metric $\M_e$ by pulling back the metric weighted force form of the geometry $\M_e\xdd + \f = \zero$, where again $\f = \M_e\h$. The pullback is
\begin{align}
    &\J^\tr\M_e\big(\J\qdd + \Jd\qd\big) + \J^\tr\f = \zero \\
    &\Rightarrow \big(\J^\tr\M_e\J\big) \qdd + \J^\tr\big(\f + \M_e\Jd\qd\big) \\\label{eqn:PullbackGeometry}
    &\Leftrightarrow \wt{\M}_e\qdd + \wt{\f} = \zero
\end{align}
where $\wt{\M}_e = \J^\tr\M_e\J$ as before and $\wt{\f} = \J^\tr \big(\f + \M_e\Jd\qd\big)$.

Let $\wt{\Lag}_e = \Lag_e\big(\phi(\q), \J\qd\big)$ be the pullback of the energy function $\Lag_e$. We know that the Euler-Lagrange equation commutes with the pullback, so applying the Euler-Lagrange equation to this pullback energy $\wt{L}_e$ is equivalent to pulling back the Euler-Lagrange equation of $\Lag_e$. This means we can calculate the Euler-Lagrange equation of $\Lag_e$ as
\begin{align}
    &\big(\J^\tr\M_e\J\big) \qdd + \J^\tr\big(\f_e + \M_e\Jd\qd\big) = \zero\\
    &\Leftrightarrow \wt{\M}_e\qdd + \wt{\f}_e = \zero,
\end{align}
with $\wt{\M}_e = \J^\tr\M_e\J$ and $\wt{\f}_e = \J^\tr\big(\f_e + \M_e\Jd\qd\big)$ (both as previously defined). Therefore, energizing \ref{eqn:PullbackGeometry} with $\wt{\Lag}_e$ gives
\begin{align}\nonumber
    &\wt{\M}_e\qdd + \wt{\f}_e 
      + \wt{\M}_e\left[
        \wt{\M}_e^{-1} - \frac{\qd\qd^\tr}{\qd^T\wt{\M}_e\qd}\right]
        \big(\wt{\f} - \wt{\f}_e\big) = \zero \\\nonumber
    &\Rightarrow \wt{\M}_e\qdd + \wt{\f}_e \\\nonumber
    &\ \ \ \ + \big(\J^\tr\M_e\J\big)
          \left[
            \wt{\M}_e^{-1} - \frac{\qd\qd^\tr}{\qd^\tr \J^\tr\M_e\J\qd}
          \right] \\\nonumber
    &\ \ \ \ \ \ \ \ \ \ \ \cdot\Big(
          \J^\tr\big(\f + \M_e\Jd\qd\big)
          - \J^\tr\big(\f_e + \M_e\Jd\qd\big)
        \Big)
      = \zero \\\nonumber
    &\Rightarrow \wt{\M}_e\qdd + \wt{\f}_e \\\nonumber
    &\ \ \ \ + \big(\J^\tr\M_e\big)
          \J\left[
            \big(\J^\tr\M_e\J\big)^{-1} - \frac{\qd\qd^\tr}{\xd^\tr\M_e\xd}
          \right] \\\nonumber
    &\ \ \ \ \ \ \ \ \ \ \ \cdot\J^\tr\big(\f- \f_e\big) = \zero\\\label{eqn:PullbackEnergizationAlmost}
    &\Rightarrow \wt{\M}_e\qdd + \wt{\f}_e \\\nonumber
    &\ \ \ \ + \J^\tr\M_e
          \left[
            \J\big(\J^\tr\M_e\J\big)^{-1}\J^\tr - \frac{\xd\xd^\tr}{\xd^\tr\M_e\xd}
          \right] \\\nonumber
    &\ \ \ \ \ \ \ \ \ \ \ \cdot \big(\f- \f_e\big) = \zero.
\end{align}
Since 
\begin{align}
    \J^\tr\M_e\J\big(\J^\tr\M_e\J\big)^{-1}\J^\tr
    = \J^\tr = \J^\tr\M_e\big(\M_e^{-1}\big),
\end{align}
we can write Equation~\ref{eqn:PullbackEnergizationAlmost} as
\begin{align}
    \wt{\M}_e\qdd + \wt{\f}_e
      + \J^\tr\M_e
          \left[
            \M_e^{-1} - \frac{\xd\xd^\tr}{\xd^\tr\M_e\xd}
          \right]
        \big(\f- \f_e\big) = \zero,
\end{align}
which matches the expression for the energized geometry pullback in Equation~\ref{eqn:EnergizationPullback}. 
\end{proof}


\begin{proposition}[Metric weighted average of geometries.] \label{prop:weighted_geometries}
Let $\x_i = \phi_i(\q)$ for $i=1,\ldots,m$ denote the star-shaped reduction of any transform tree, and suppose the leaves are populated with geometries $\xdd_i + \h_{2,i} = \zero$ with Finsler energies $\Lag_{e_i}$ with energy tensors $\M_i = \partial^2_{\xd\xd}\Lag_{e_i}$. Then the metric weighted pullback of the full leaf geometry is $\qdd + \wt{\h}_2 = \zero$,
with 
\begin{align}
    \wt{\h}_2 = \left(\sum_{i=1}^m\wt{\M}_i\right)^{-1}\sum_{i=1}^m \wt{\M}_i\wt{\h}_{2,i},
\end{align}
where $\wt{\M}_i = \J^\tr\M_i\J$ and $\wt{\h}_{2,i} = \wt{\M}_i^{\dagger}\J^\tr\M_i\big(\h_{2,i} - \Jd\qd\big)$ are the standard pullback components written in acceleration form.
\end{proposition}
\begin{proof}
The standard algebra on $(\M_i\f_i)$ holds, where $\f_i = \M_i\h_{2,i}$. Pulling back gives $(\wt{\M}_i,\wt{\f}_i)$ and summing gives
$\sum_i (\wt{\M}_i,\wt{\f}_i) = \big(\sum_i\wt{\M}_i, \sum_i\wt{\f}_i\big)$.
Expressing that result in canonical form gives
\begin{align}
    \left(\sum_i\wt{\M}_i, \Big(\sum_i\wt{\M}_i\Big)^{-1}\sum_i\wt{\M}_i\wt{\h}_{2,i}\right),
\end{align}
where $\wt{\h}_{2,i}$ is the acceleration form of the individual pullbacks. Expanding gives the formula.

Since $\Lag_{e_i}$ are Finsler energies, the pullback metrics $\wt{\M}_i$ are homogeneous of degree 0 in velocity (i.e. they depend only on the normalized velocity $\hat{\xd}$). Therefore, $\wt{\h}_2$ is homogeneous of degree 2 and the pullback forms a geometry generator.
\end{proof}



Many of the above classes of fabric introduced above (possibly all of them, see below) are closed under the spec algebra. Specifically, we say that a class of fabrics is closed if a spec from that class remains in the same class under the spec algebra operations of combination and pullback.

\begin{theorem} \label{thm:ClosureUnderSpecAlgebra}
The following classes of fabrics are known to be closed under the spec algebra: Lagrangian fabrics, Finsler fabrics, and geometric fabrics.
A fabric of each of these types will remain a fabric of the same type under spec algebra operations in regions where the differentiable transforms are full rank and finite.
\end{theorem}
\begin{proof}
Lagrangian systems are covariant (see \cite{ratliff2020SpectralSemiSprays}) and the Euler-Lagrange equation commutes with pullback. Since full rank finite differentiable transforms define submanifold constraints and diffeomorphisms on those constraints, Lagrangian systems are in general closed under pullback. Moreover, the homogeneity of the Finsler structure is preserved under differentiable map composition since the differentiable map is independent of velocity, so Finsler systems, themselves, are closed under pullback. Both Lagrangian and Finsler systems are individually closed under Cartesian product, so since any collection of spec algebra operations can be expressed as a pullback from a Cartesian product space to a root space \cite{ratliff2020SpectralSemiSprays}, Lagrangian and Finsler systems are closed under the spec algebra.

For geometric fabrics, if we can show that geometry generators are closed under the spec algebra, then since Finsler energies are as well and geometric fabrics are defined by energization, geometric fabrics themselves are closed under the spec algebra. To examine the closure of geometry generators, we note that the pullback $\wt{\h}_2 = \J^\tr\M\big(\h_2 + \Jd\qd\big)$ of a geometry generator $\h_2(\x, \xd)$ with respect to a Finsler metric $\M(\x, \xd)$ under differentiable map $\x = \phi(\q)$ remains homogeneous of degree 2 since the Finsler metric is homogeneous of degree 0 and the curvature term $\Jd\qd = \big(\partial_{\x\x}^2\phi\,\qd\big)\qd$ is homogeneous of degree 2. Therefore, it's a geometry generator by definition.
\end{proof}
\fi



\section{Speed control via execution energy regulation}
\label{sec:speed_control}

Once a geometry is energized by a Finsler energy $\Lag_e$, its speed profile is defined by $\Lag_e$. In practice, we usually want to regulate a different energy, an {\em execution energy} $\Lag_e^{\mathrm{ex}}$, such as Euclidean energy. We can regulate the speed using
\begin{align} \label{eqn:FundamentalsOfSpeedControl}
    \xdd = -\M_e^{-1}\partial_\x\psi(x) + \pi_0(\x, \xd) + \alpha_\mathrm{reg} \xd,
\end{align}
as long as $\alpha_\mathrm{reg} < \alpha_{\Lag_e}$ (see Theorem~\ref{thm:EnergizedFabrics}), where $\alpha_{\Lag_e}$ is the energization coefficient (Equation~\ref{eqn:EnergizationTransformAlpha}) for the fabric's energy $\Lag_e$.

Let $\Lag_e^\mathrm{ex}$ be an {\em execution} energy, which may differ from the fabric's energy $\Lag_e$. Let $\alpha_\mathrm{ex}^0$ and $\alpha_\mathrm{ex}^\psi$ be energization coefficients, respectively, for $\mathrm{energize}_{\Lag_e^\mathrm{ex}}\big[\pi\big]$ and $\mathrm{energize}_{\Lag_e^\mathrm{ex}}\big[-\M_e^{-1}\partial_\x\psi + \pi\big]$. Likewise, let $\alpha_{\Lag_e}$ denote the energy coefficient of $\mathrm{energize}_{\Lag_e}\big[\pi\big]$. For speed regulation, we use
\begin{align}
    \alpha_\mathrm{reg} 
    = \alpha_\mathrm{ex}^\eta - \beta_\mathrm{reg}(\x, \xd) + \alpha_\mathrm{boost}
\end{align}
where $\alpha_\mathrm{ex}^\eta = \eta \alpha_\mathrm{ex}^0 + (1-\eta)\alpha_\mathrm{ex}^\psi$ for $\eta\in[0,1]$, and $\alpha_\mathrm{boost} \leq 0$ can be used temporarily in the beginning to boost the system up to speed. Since it is transient, we drop $\alpha_\mathrm{boost}$ momentarily for simplicity, but return to it at the end of the section to describe a good boosting policy. 
Under this choice, using $\vv_\paral = \big(\alpha_\mathrm{ex}^0 - \alpha_\mathrm{ex}^\psi\big)\xd$ to denote the component of $-\M_e^{-1}\partial_\x\psi$ additionally removed by including it within the energization operation (and dropping $\alpha_\mathrm{boost}$), we can express the system as
\begin{align}
    \xdd &= \eta \vv_\paral + \mathrm{energize}_{\Lag_e^\mathrm{ex}}\big[\pi\big] - \beta_\mathrm{reg}(\x, \xd)\xd,
\end{align}
where $\beta_\mathrm{reg} > 0$ is a strictly positive damper. Here $\alpha_\mathrm{reg} = \alpha_\mathrm{ex} - \beta_\mathrm{reg} < \alpha_{\Lag_e}$ satisfying the bound of Equation~\ref{eqn:FundamentalsOfSpeedControl} required for stability. $\beta$ removes energy from the system, and by adjusting $\eta$ we can regulate how much energy is injected into the system from the potential. We choose $\beta_\mathrm{reg} = s_\beta(\x) B + \underline{B} + \max\{0, \alpha_\mathrm{ex}^\eta - \alpha_{\Lag_e}\} > 0$, using $\underline{B} > 0$ as a constant baseline damping coefficient, and $s_\beta(\x) B$ for additional damping near the convergence point with $B>0$ constant and $s_\beta(\x)$ acting as a switch transitioning from 0 to 1 as the system approaches the target. $\max\{0, \alpha_\mathrm{ex}^\eta - \alpha_{\Lag_e}\}$ ensures the stability bound $\alpha_\mathrm{reg} = \alpha_\mathrm{ex} - \beta_\mathrm{reg} < \alpha_{\Lag_e}$ is satisfied.

We use $\alpha_\mathrm{boost}$ to explicitly inject energy along the direction of motion to quickly boost the system up to speed. When $\xd = \zero$, this term has no effect, so the initial direction of motion is chosen by the potential $\-\partial_\x\psi$. Once $\xd \neq \zero$, $\alpha_\mathrm{boost}$ quickly accelerates the system to the desired speed. When $\pi$ is a geometry, quickly reaching a high speed means that the influence of the non-geometric potential is diminished, promoting path consistency. See Section VIIC in supplemental paper \cite{xiegeometricfabricsupp} for a concrete application of speed control.

\section{Empirical demonstration with a planar arm} \label{sec:Experiments}


This section presents some experiments demonstrating this theoretical framework in practice on a 2D planar arm. These experiments provide a proof of concept of the theoretical framework outlined here using geometric fabrics. We build a simple tabletop reaching primitive layer-by-layer exploiting the geometric consistency of geometric fabrics. The theory characterizes the relationships between geometric fabrics and classical energy shaping techniques (geometric fabrics are formed by bent Finsler geometries, which are a type of Lagrangian system, which in turn generalizes the classical mechanical systems governing operational space control; they enable velocity dependence in the metrics and extend the geometric consistency of Finsler systems to separate priority metric design from geometric policy design). We do not perform a thorough analysis of the relative advantages of this added flexibility here. We show simply that we can use it relatively easily and we show the type of behavior we can get from it; a thorough empirical investigation has been performed, but it is in a separate work currently under peer review \cite{xiegeometricfabricsupp} (included in the supplementary material).

The basic principles behind geometric fabric design is to: 1) select a collection of relevant task spaces; 2) in each space, define one or more fabric terms each consisting of a homogeneous of degree 2 (HD2) geometry generating equation representing an {\em acceleration policy} (what it wants to do in that space) and a corresponding Finsler energy defining its {\em priority metric} (defining its priority relative to other terms---the full behavior will be a metric weighted average of of the policies); 3) choose a potential function to drive the system to its task goal; 4) control the system using the equations of speed control derived in Section~\ref{sec:speed_control}. 

Often the task spaces are natural for the problem. For instance, for this reaching task we care mostly about end-effector control, so the end-effector space and/or the space defined by distance to the goal are key relevant task space. Another common task space, also used here, is the space of joint angles, or the derived space defined by the distance to a given joint's joint limit. 

Step 2 might sound daunting, but defining geometry generators and Finsler energies is actually relatively straightforward. We need only create a function with the right properties, the key property being its homogeneity of degree 2 (in both cases). There are many ways to construct such a function; an easy way is to start with a function that has no velocity dependence at all and then simply multiply it by the squared norm of the velocity (the second degree homogeneity comes entirely from the velocity factor in this case). 

Likewise, we can multiply by a factor that is a function of only the normalized velocity (i.e. its directionality, which is homogeneous of degree 0) without affecting the overall homogeneity of the function. In one-dimension, such as in the case of a joint limit space, this allows us to have a switch that turns a geometry on and off based on the sign of the velocity (since the normalized 1D velocity is its sign). Such a factor can, for instance, turn the term on when it is moving toward a limit and turn it off when moving away. Since the overall policy is homogeneous of degree 2, although the metric ends up with a discrete switch, the stability analysis holds and the policy changes smoothly across the switch moving through zero acceleration (meaning the corresponding term drops out of the metric weighted average at zero velocity).

\begin{figure}[t]
  \includegraphics[width=.98\columnwidth]{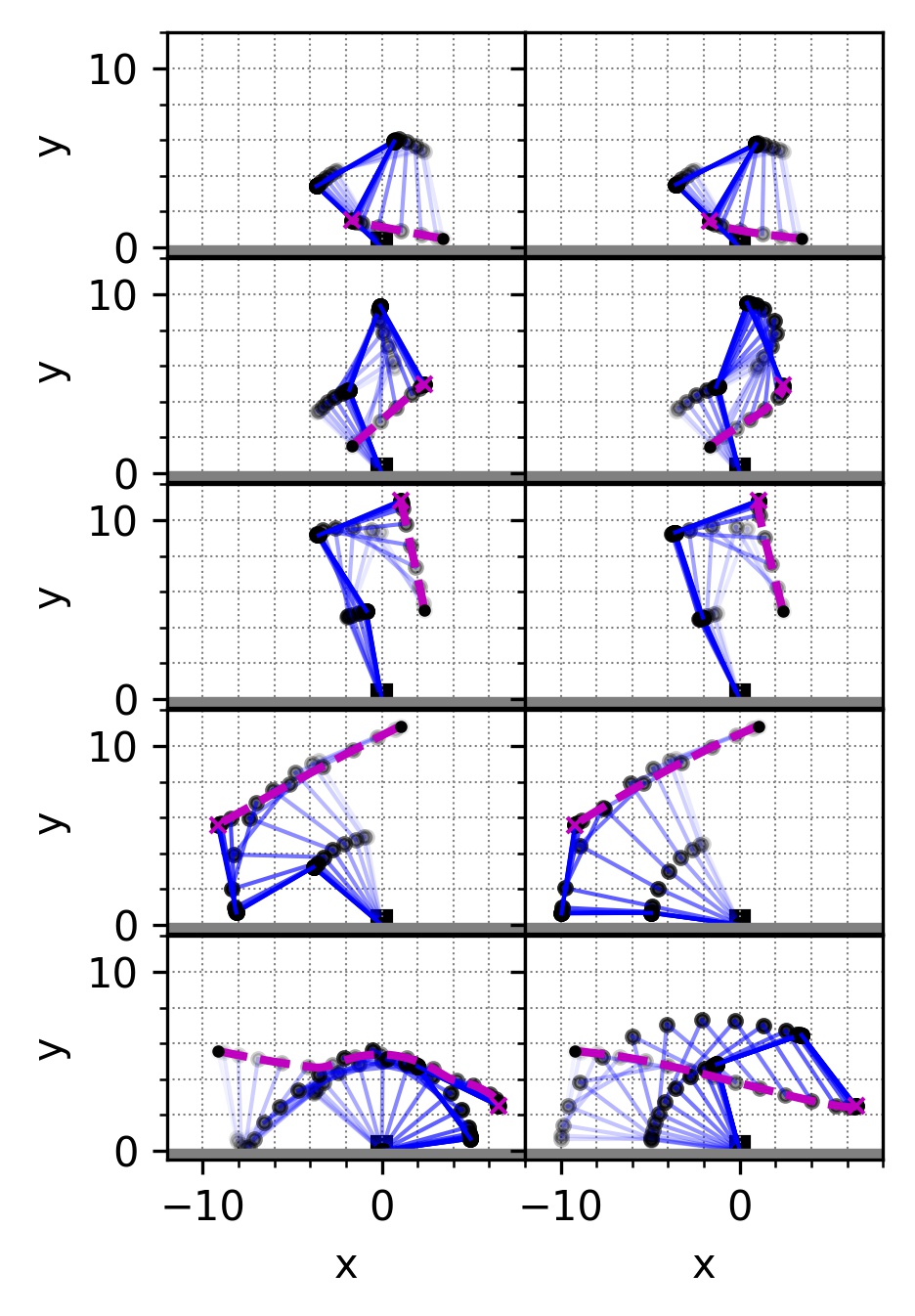}
  \vspace{-20pt}
  \caption{Continuous planar goal reaching experiment. The left and right columns shows the behavior with and without the redundancy resolution geometries, respectively. The arm's alpha transparency ranges from light at the beginning of the trajectory to dark at the end. Note in particular the last two rows: without redundancy resolution the arm bends over and ultimately collapses on itself, while with redundancy resolution it successfully retains its overall ``ready'' configuration.}
  \label{fig:5goals_compare}
\end{figure}

We use four types of fabric terms for these experiments. 1) a goal directing behavior at the end-effector;  2) joint limit avoidance; 3) redundancy resolution; 4) end-effector behavior shaping. Our experiments build up the full behavior layer-by-layer, adding first just the goal attractor and joint limit avoidance. We compare that to the same system with redundancy resolution added, and then to the resulting system with added end-effector behavior shaping. Separately, many of the earlier components can be used as components for many manipulation problems, and together they constitute a natural primitive for tabletop manipulation. The system is driven by a simple attractor injecting energy to pull the system toward the target. (In that sense, goal directing behavior is technically unnecessary, although it helps shape the end-effector's path toward the target.)


Geometric consistency enables us to straightforwardly tune each part of the system one layer at a time rather than all components together. With each new layer, we fix the previous layers and simply add the new one. This feature significantly reduces design complexity. Figure~\ref{fig:5goals_compare} (left) shows the first behavior, including only end-effector direction and joint limit avoidance; Figure~\ref{fig:5goals_compare} (right) contrasts that to the same behavior with added redundancy resolution. Note that these behaviors are similar in nature to classical operational space control methods. Fabrics, though, aren't restricted to hierarchies (they use soft priorities to trade off relevant preferences), and each term is more expressive. For instance, operational space control is restricted to position dependent metrics and potentials; here even joint limit avoidance goes beyond that, including velocity dependence enabling it to consider the obstacle while close to it and heading toward it, but to ignore it while far from the obstacle and/or moving away from it. Despite this added design flexibility, fabrics are remain provably stable.

Finally, Figure~\ref{fig:5goals_behavior_shaping} shows the full reaching behavior with end-effector behavior shaping lifting from the surface, traveling toward the target at a nominal distance above the surface, and then descending to the target from above.

\begin{figure}
  \vspace{-20pt}
  \includegraphics[width=.7\columnwidth]{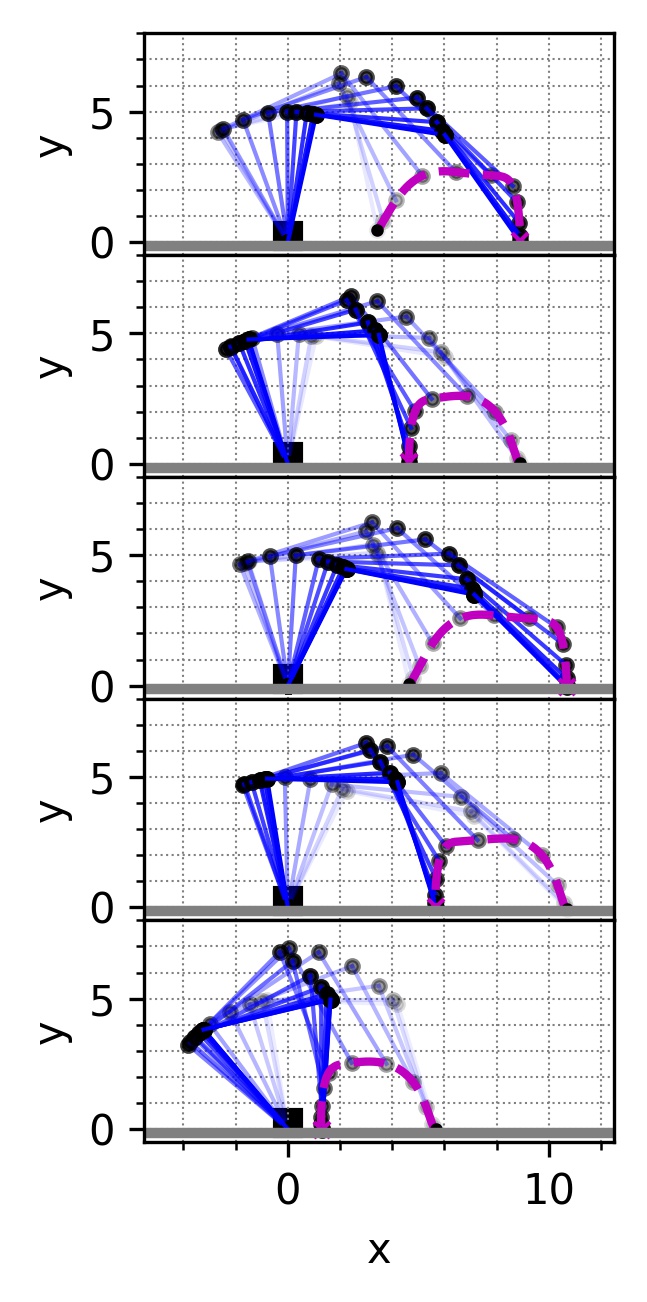}
  \vspace{-20pt}
  \caption{Planar robot behavior shaping experiment with geometric fabrics. From the top to the bottom, figures show the robot moving sequentially from goal 1 to goal 5, each time starting from the final configuration of the last episode. The arm's alpha transparency ranges from light at the beginning of the trajectory to dark at the end. As designed, geometric fabrics on the end-effector shape the system to first lift the end-effector from the floor, then proceed horizontally a fixed distance above the floor before lowering to the target from above. }
  \label{fig:5goals_behavior_shaping}
\end{figure}

\section{CONCLUSIONS}

We present our theory of fabrics ranging from broad characterization to concrete tools.
Fabrics illuminate limitations of earlier tools such as operational space control \cite{KhatibOperationalSpaceControl1987,Peters_AR_2008} and geometric control \cite{bullo2004geometric} (including our own Geometric Dynamical Systems (GDS) \cite{cheng2018rmpflow}), which as Lagrangian fabrics have limited capacity. We see that, despite lacking theoretical guarantees, the broader RMPs \cite{ratliff2018rmps} are fundamentally more expressive. To bridge the gap, we derive geometric fabrics, our most concrete incarnation of fabrics, for the stable and geometrically consistent design of speed-regulated RMPs. This paper is highly technical, but its product is simple and safe tools for behavioral design. We provide specifics for practitioners in our supplemental paper on geometric fabrics \cite{xiegeometricfabricsupp} along with a suite of realistic experiments on a full implementation. 




{
\bibliographystyle{plainnat}
\bibliography{refs}
\clearpage
}

\begin{appendices}
\title{Optimization fabrics for behavioral design: Appendices}
\date{}
\maketitle

\section{Concrete derivations on manifolds}
\label{apx:Manifolds}

Nonlinear geometry is most commonly constructed in terms of smooth manifolds, often in an abstract, coordinate-free, form. To make the topic more accessible, we will stick to coordinate descriptions and standard vector notations from advanced calculus. Analogous to how standard classical equations of motion are expressed in generalized coordinates convenient to the problem (Cartesian coordinates, polar coordinates, robotic joint angles, etc.) and understood to represent concrete physical phenomena independent of those coordinates, we take the same model here. Our constructions of nonlinear geometry will be made exclusively in terms of coordinates to keep the expressions and notation familiar, and practitioners are free to change coordinates as needed as the system moves across the manifold.  

Formally, the equations we describe in this paper are \textit{covariant}, which means they maintain their form under changes of coordinates \cite{LeeSmoothManifolds2012}. 
Instead of ensuring explicitly all objects used in equations are coordinate free, we define quantities in terms of clearly coordinate free quantities such as lengths. For instance, we define geometry in terms of minimum length criteria, so as long as the length measure transforms properly so it remains consistent under changes of coordinates, the geometric equations should be independent of coordinates.

Manifolds will be defined in the traditional way (see \cite{LeeSmoothManifolds2012} for a good introduction), but for our purposes, we will consider them $d$-dimensional spaces $\mathcal{X}$ with elements identified with $\x\in\R^d$ in $d$ coordinates. Often we implicitly assume a system evolves over time $t$ in a trajectory $\x(t)$ with velocity $\xd = \frac{d\x}{dt}$, and say that coordinate velocity vector $\xd$ is an element of the tangent space $\mathcal{T}_\x\mathcal{X}$. When discussing general tangent space vectors, we often us a separate notation $\vv\in\mathcal{T}_\x\mathcal{X}$ to distinguish it from being a velocity of a specific trajectory.

Manifolds with a boundary are common modeling tools in this work as well. For instance, both joint limits and obstacles form boundaries in a manifold. If $\mathcal{X}$ denotes a manifold, its boundary is denoted $\partial\mathcal{X}$ and is assumed to form a smooth lower-dimensional submanifold with cooresponding tangent space $\mathcal{T}_\x\partial\mathcal{X}$. 

\section{Advanced calculus notation}
\label{apx:CalculusNotation}

Higher-order tensors are common in differential geometry, so frequently when dealing with coordinates, explicit indices are exposed and summed over to handle the range of possible combinations. To avoid clutter, the Einstein convention is then used to unambiguously drop summation symbols. 

However, the resulting index notation is unfamiliar and takes some getting used to. Alternatively, the matrix-vector notation found in many advanced calculus and engineering texts is often much simpler and concise. The key to this notation's simplicity is the associativity of matrix products. When at most two indices are involved, by arranging the components of our indexed objects into matrices, with one-indexed vectors being column vectors by default, we can leverage this associativity of matrix products to remove the indices entirely while ensuring expression remain unambiguous with regard to order of operation. We, therefore, use the simpler and more compact matrix-vector notation wherever possible, with a slight extension for how to deal with added indices beyond 2 that might arise from additional partials as discussed below. The majority of our algebra can be expressed using just two indices allowing us to remain within the matrix-vector paradigm.
 
Whenever we take a partial derivative, a new index is generated ranging over the individual dimensions of the partial derivative. For instance, a partial derivative of a function $\partial_\x f(\x)$ where $f:\R^n\rightarrow\R$ is a list of $n$ partial derivatives, one for each dimension of $\x$. Likewise, if $\g:\R^n\rightarrow\R^m$, there are again $n$ partials in $\partial_\x\g$, but this time each of them has the same dimensionality as the original $\g$. $\g$ has its own index ranging over its $m$ codomain dimensions, and now there's a new second index ranging over the $n$ partials. 

We use the convention that if the partial derivative generates a first index (as in $\partial_\x f$), it is oriented as a column vector by default (in this case $n$ dimensional), with its transpose being a row vector. If it generates a second index (as in $\partial_\x \g$), the second index creates a matrix (in this case $m\times n$ dimensional) so that the original vector valued function's orientation is maintained. Specifically, if originally the vector-valued function was column oriented, then each partial will be column oriented and they will be lined up in the matrix so that the new index ranges over the columns. And if the original vector-valued function is row oriented, then the partials will be row vectors and stacked to make rows of a matrix, so the new index will range over the first index of the resulting matrix.

We use the compact notation $\partial_\x$ rather than $\frac{\partial}{\partial \x}$ so multiple partial derivatives unambiguously lists the partials in the order they're generating indices. The notation $\partial_{\x\y} h(\x,\y)$ where $h:\R^m\times\R^n\rightarrow\R$, therefore, means we first generate an index over partials of $\x$ and then generate an index over partials of $\y$. That means $\partial_{\x\y}h = \partial_\y\big(\partial_\x h\big)$ is an $m\times n$ matrix since $\partial_\x h$ is first generated as an $n$-dimensional column vector. When the partials are over two of the same variable denoting a Hessian $\mH_f$, we use the notation $\mH_f = \partial^2_{\x\x} f$ using a squared exponent to emphasize that the partials are not mixed.

For partials generating up to only the first two indices, matrix algebra governs how they interact with surrounding matrices and column or row vectors, with the partial derivative operator taking priority over matrix multiplication (we use parentheses when the product should be included in the partial). For instance, $\partial_\x \g\:\xd = \J_\g\xd$ is the Jacobian of $\g$, with partials ranging over the columns (second index) and each partial constituting a full column, multiplied by the column vector $\xd$. Likewise, $\frac{1}{2}\xd^\tr\partial^2_{\xd\xd}\Lag_e\xd$ is a squared norm of $\xd$ with respect to the symmetric matrix $\partial^2_{\xd\xd}\Lag_e = \partial_\xd\partial_\xd\Lag_e$ of second partials. The first $\partial_\xd\Lag_e$ creates a column vector (it's a first index), and the second $\partial_\xd\big(\partial_\xd\Lag_e\big)$ creates a matrix, with columns constituting the partials of the vector of first partials.

We use bold font to emphasize that these objects with one or two indices are treated as matrices, with one index objects assumed to be column oriented by default. Transposition operates as standard in matrix algebra, swapping the indices in general, with column vectors becoming row vectors and vice versa.

Beyond two indices, we use the convention that a partial derivative simply generates a new index ranging over the partials. By default, multiplication on the right by another indexed object (e.g. vector or matrix) will contract (sum products of matching index values) across this new index and the first index in the right operand. For instance, if $\M:\R^d\rightarrow\R^{m\times n}$, and $\qd\in\R^d$ is a velocity, then $\partial_\q\M\:\qd = \sum_{i=1}^d\frac{\partial\M}{\partial q^i} \dot{q}^i$, unambiguously. Note that beyond two indices, associativity no longer holds in general, so parentheses must be used to disambiguate wherever necessary.

\vspace{10pt}
\noindent Examples:
\begin{enumerate}
\item $\partial_\x f(\x,\y)$ is a column vector
\item $\partial_{\x\y} f(\x,\y)$ is a matrix with first index ranging over partials of $\x$ and second index ranging over partials of $\y$.
\item When $\f:\R^n\rightarrow\R^m$ the matrix $\partial_\x \f$ is an $m\times n$ Jacobian matrix, and $\partial_\x \f^\tr$ is an $n\times m$ matrix, the transpose of the Jacobian matrix.
\item $\frac{d}{dt} \f(\x) = \partial_\x \f\:\xd$ and $\frac{d}{dt} \f^\tr(\x) = \xd^\tr \partial_\x \f^T = \xd^\tr \big(\partial_\x \f\big)^\tr$.
\item $\partial_\x g\big(\f(\x)\big) = \big(\partial_\x \f\big)^\tr \partial_\y g|_{\f(\x)} = J_\f^\tr \partial_\y g(\f(\x))$ with $\y = \f(\x)$.
\end{enumerate}
Often a gradient is denoted $\nabla_\x f(\x)$, but with these conventions outlined above, we use simply $\partial_\x f(\x)$ to avoid redundant notation. We additionaly frequently name common expressions for clarity, such as 
\begin{enumerate}
    \item $\p_e = \partial_\xd\Lag_e$
    \item $\M_e = \partial^2_{\xd\xd}\Lag_e = \partial_\xd \p_e$
    \item $\J_{\p_e} = \partial_\x \p_e$ (even though $\p_e$ is a function of both position and velocity, we use $\J$ to denote the position Jacobian)
    \item $\g_e = \partial_\x\Lag$
\end{enumerate}
These vector-matrix definitions $\M_e,\J_{\p_e},\g_e$ make it more clear what the general size and orientation of the objects $\partial^2_{\xd\xd}\Lag, \partial_{\xd\x}\Lag = \partial_\x \p_e$, and $\partial_\x\Lag$ are.

The equations of motion in different forms would be:
\begin{align}
    &\partial^2_{\xd\xd}\mathcal{L}\;\xdd +  \partial_{\xd\x}\mathcal{L}\:\xd - \partial_\x \mathcal{L} = \zero.\\
    &\M_e\xdd + \J_{\p_e}\xd - \g_e = \zero.
\end{align}
Using the named quantities, that structure of the final expression is clear at a glance. Similarly, we have
\begin{align}
E(\x,\xd) &= \frac{1}{2}\xd^\tr\:\partial^2_{\xd\xd}\mathcal{L}\: \xd = \frac{1}{2}\xd^\tr\M_e\xd,
\end{align}

Example of algebraic operations using this notation (taken from a common calculation involving the time invariance of Finsler energy functions):
\begin{align}
    &\xd^\tr \partial_{\xd\x}\Lag\:\xd - 2\partial_\x\Lag_e^\tr\xd
    = \big( \xd^\tr \partial_\x\partial_\xd\Lag_e - 2\partial_\x \Lag_e^\tr \big)\xd\\
    &\ \ = \big( \partial_\x(\xd^\tr \partial_\xd\Lag_e)^\tr - 2\partial_\x \Lag_e^\tr \big)\xd\\
    &\ \ = \partial_\x\big(\xd^\tr\partial_\xd \Lag_e - 2\Lag_e\big)^\tr \xd \\
    &\ \ = \partial_\x\big(2\Lag_e - 2\Lag_e\big)^\tr\xd\\
    &\ \ = \zero,
\end{align}
where we use $\Lag_e = \frac{1}{2}\xd^\tr\partial_\xd\Lag_e$ in that last line.

\section{Additional Lemmas and proofs from conservative fabrics section}

\begin{lemma}\label{lma:EnergyTimeDerivative}
Let $\Lag_e$ be an energy Lagrangian with energy $\Ham_e = \partial_\xd\Lag_e^\tr \xd - \Lag_e$. The energy time derivative is 
\begin{align} \label{eqn:EnergyTimeDerivative}
    \dot{\Ham}_e = \xd^\tr\big(\M_e\xdd + \f_e\big),
\end{align}
where $\M_e$ and $\f_e$ come from the Lagrangian's equations of motion $\M_e\xdd + \f_e = \zero$.
\end{lemma}
\begin{proof}
The calculation is a straightforward time derivative of the Hamiltonian:
\begin{align*}
    \dot{\Ham}_{\Lag_e} 
    &= \frac{d}{dt}\big[\partial_\xd\Lag_e - \Lag_e\big] \\
    &= \big(\partial_{\xd\xd}\Lag_e\,\xdd + \partial_{\xd\x}\Lag_e\,\xd\big)^\tr \xd + \partial_\xd\Lag_e^\tr\,\xdd \\
    &\ \ \ \ \ \ \ \ - \big(\partial_\xd\Lag_e\,\xdd + \partial_\x\Lag_e^\tr\xd \big) \\
    &= \xd^\tr\Big(\partial_{\xd\xd}\Lag_e\,\xdd + \partial_{\xd\x}\Lag_e\,\xd - \partial_\x\Lag_e^\tr\xd\Big) \\
    &= \xd^\tr\big(\M_e\xdd + \f_e\big).
\end{align*}
\end{proof}

The following lemma collects some results around common matrices and operators that arise when analyzing energy conservation.
\begin{lemma}\label{lma:EnergyProjection}
Let $\Lag_e$ be an energy Lagrangian. Then with $\p_e = \M_e\xd$,
\begin{align}
    \mR_{\p_e} = \M_e^{-1} - \frac{\xd\,\xd^\tr}{\xd^\tr\M_e\xd}
\end{align}
has null space spanned by $\p_e$ and
\begin{align}
    \mR_\xd = \M_e - \frac{\p_e\p_e^\tr}{\p_e^\tr\M_e^{-1}\p_e}
\end{align}
has null space spanned by $\xd$. These matrices are related by
$\mR_\xd = \M_e\mR_{\p_e}\M_e$ and the matrix $\M_e^{-1}\mR_\xd = \M_e\mR_{\p_e} = \mP_e$ is a projection operator of the form
\begin{align} \label{eqn:EnergyProjector}
    \mP_e = \M_e^{\,\frac{1}{2}}\Big[\I - \hat{\vv} \hat{\vv}^\tr\Big]\M_e^{-\frac{1}{2}}
\end{align}
where $\vv = \M_e^{\,\frac{1}{2}}\xd$ and $\hat{\vv} = \frac{\vv}{\|\vv\|}$ is the normalized vector. Moreover, $\xd^\tr\mP_e\f = \zero$ for all $\f(\x,\xd)$.
\end{lemma}
\begin{proof}
Right multiplication of $\mR_{\p_e}$ by $\p_e$ gives:
\begin{align}
    \mR_{\p_e} \p_e 
    &= \left(\M_e^{-1} - \frac{\xd\,\xd^\tr}{\xd^\tr\M_e\xd}\right) \M_e\xd \\
    &= \xd - \xd \left(\frac{\xd^\tr\M_e\xd}{\xd^\tr\M_e\xd}\right) = \zero,
\end{align}
so $\p_e$ lies in the null space. Moreover, the null space is no larger since each matrix is formed by subtracting off a rank 1 term from a full rank matrix. 

The relation between $\mR_{\p_e}$ and $\mR_\xd$ can be shown algebraically
\begin{align}
    \M_e\mR_{\p_e}\M_e 
    &= \M_e\left(\M_e^{-1} - \frac{\xd\,\xd^\tr}{\xd^\tr\M_e\xd}\right)\M_e \\
    &= \M_e - \frac{\M_e\xd\,\xd^\tr\M_e}{\xd^\tr\M_e\xd} \\
    &= \M_e - \frac{\p_e\p_e^\tr}{\p_e^\tr\M_e^{-1}\p_e} = \mR_\xd.
\end{align}
Since $\M_e$ has full rank, $\mR_\xd$ has the same rank as $\mR_{\p_e}$ and its null space must be spanned by $\xd$ since $\mR_\xd\xd = \M_e\mR_{\p_e}\M_e\xd = \M_e\mR_{\p_e}\p_e = \zero$.

With a slight algebraic manipulation, we get
\begin{align*}
    \M_e\mR_{\p_e} 
    &= \M_e\left(\M_e^{-1} - \frac{\xd\,\xd^\tr}{\xd^\tr\M_e\xd}\right) \\
    &= \M_e^{\,\frac{1}{2}} \left(\I - \frac{\M_e^{\,\frac{1}{2}}\xd\,\xd^\tr\M_e^{\,\frac{1}{2}}}{\xd^\tr\M_e^{\,\frac{1}{2}}\M_e^{\,\frac{1}{2}}\xd}\right) \M_e^{-\frac{1}{2}} \\
    &= \M_e^{\,\frac{1}{2}} \left(\I - \frac{\vv\vv^\tr}{\vv^\tr\vv}\right) \M_e^{-\frac{1}{2}} \\
    &= \mP_e
\end{align*}
since $\frac{\vv\,\vv^\tr}{\vv^\tr\vv} = \hat{\vv}\hat{\vv}^\tr$. Moreover, 
\begin{align*}
    \mP_e\mP_e 
    &= 
    \M_e^{\,\frac{1}{2}} \left(\I - \frac{\vv\vv^\tr}{\vv^\tr\vv}\right) \M_e^{-\frac{1}{2}}\\
    &\ \ \ \ \ \ \cdot
    \M_e^{\,\frac{1}{2}} \left(\I - \frac{\vv\vv^\tr}{\vv^\tr\vv}\right) \M_e^{-\frac{1}{2}} \\
    &= \M_e^{\,\frac{1}{2}} \left(\I - \frac{\vv\vv^\tr}{\vv^\tr\vv}\right) \left(\I - \frac{\vv\vv^\tr}{\vv^\tr\vv}\right) \M_e^{-\frac{1}{2}} \\
    &= \M_e^{\,\frac{1}{2}} \left(\I - \hat{\vv}\hat{\vv}^\tr\right) \M_e^{-\frac{1}{2}} = \mP_e,
\end{align*}
since $\mP_\perp = \I - \hat{\vv}\hat{\vv}^\tr$ is an orthogonal projection operator. Therefore, $\mP_e^2 = \mP_e$ showing that it is a projection operator as well.

Finally, we also have 
\begin{align*}
    \xd^\tr\mP_e 
    &= \xd^\tr\M_e\mR_{\p_e}
    = \xd^\tr\M_e\left(\M_e^{-1} - \frac{\xd\,\xd^\tr}{\xd^\tr\M_e\xd}\right)\\
    &= \left[\xd^\tr - \left(\frac{\xd^\tr\M_e\xd}{\xd^\tr\M_e\xd}\right)\xd^\tr\right]
    = \zero.
\end{align*}
Therefore, for any $\f$, we have $\xd^\tr\mP_e\f = \zero$. 
\end{proof}

\begin{proposition}[System energization] \label{prop:SystemEnergization}
Let $\xdd + \h(\x, \xd) = \zero$ be a differential equation, and suppose $\Lag_e$ is any energy Lagrangian with equations of motion $\M_e\xdd + \f_e = \zero$ and energy $\Ham_e$. Then $\xdd + \h(\x, \xd) + \alpha_{\Ham_e}\xd = \zero$ is energy conserving when
\begin{align}\label{eqn:EnergizationTransformAlpha}
    \alpha_{\Ham_e} = -(\xd^\tr\M_e\xd)^{-1}\xd^\tr\big[\M_e\h - \f_e\big],
\end{align}
and differs from the original system by only an acceleration along the direction of motion. 
The new system can be expressed as:
\begin{align} \label{eqn:ZeroWorkEnergizationForm}
  \M_e\xdd + \f_e + \mP_e\big[\M_e\h - \f_e\big] = \zero.
\end{align}
This modified system is known as an {\em energized system}. 
\fullversiononly
Moreover, the operation of {\em energizing} the original system is known as an {\em energization transform} and is denoted using spec notation as
\begin{align}
    \mathcal{S}_{\h}^{\Lag_e} 
    &= \Big(\M_e, \f_e + \mP_e\big[\M_e\h - \f_e\big]\Big)_\mathcal{X} 
    \\\nonumber
    &= \mathrm{energize}_{\Lag_e}\Big\{\mathcal{S}_\h\Big\},
\end{align}
where $\mathcal{S}_\h = \big(\I, \h\big)_\mathcal{X}$ is the spec representation of $\xdd + \h(\x,\xd) = \zero$.
\fi 
\end{proposition}
\begin{proof}
Equation~\ref{eqn:EnergyTimeDerivative} of Lemma~\ref{lma:EnergyTimeDerivative} gives the time derivative of the energy. Substituting a system of the form $\xdd = -\h(\x,\xd) - \alpha_{\Ham_e}\xd$, setting it to zero, and solving for $\alpha_{\Ham_e}$ gives
\begin{align}
&\dot{\Ham}_e = \xd^\tr\Big[\M_e\big(-\h - \alpha_{\Ham_e}\xd\big) + \f\Big] = \zero \\
&\Rightarrow -\xd^\tr\M_e\h - \alpha_{\Ham_e} \xd^\tr\M_e\xd + \xd^\tr\f_e = \zero \\
&\Rightarrow \alpha_{\Ham_e} = -\frac{\xd^\tr\M_e\h - \xd^\tr\f_e}{\xd^\tr\M_e\xd}\\
&\ \ \ \ \ \ \ \ \ \ = 
-(\xd^\tr\M_e\xd)^{-1}\xd^\tr\big[\M_e\h - \f_e\big].
\end{align}
This result gives the formula for Equation~\ref{eqn:EnergizationTransformAlpha}. Substituting this solution for $\alpha_{\Ham_e}$ back in gives
\begin{align}
    \xdd 
    &= -\h + \left(\frac{\xd^\tr}{\xd^\tr\M_e\xd}\big[\M_e\h - \f_e\big]\right) \xd \\
    &= -\h + \left[\frac{\xd\,\xd^\tr}{\xd^\tr\M_e\xd}\right]\big(\M_e\h - \f_e\big).
\end{align}
Algebraically, it helps to introduce $\h = \M_e^{-1}\big[\f_f + \f_e\big]$ to first express the result as a difference away from $\f_e$; in the end we will convert back to $\h$. Doing so and moving all the terms to the left hand side of the equation gives
\begin{align*}
    &\ \ \ \ \xdd + \h - \left[\frac{\xd\,\xd^\tr}{\xd^\tr\M_e\xd}\right]\big(\M_e\h - \f_e\big) = \zero \\
    &\Rightarrow \xdd + \M_e^{-1}\big[\f_f + \f_e\big] \\
    &\ \ \ \ \ - \left[\frac{\xd\,\xd^\tr}{\xd^\tr\M_e\xd}\right]\Big(\M_e\M_e^{-1}\big[\f_f + \f_e\big] - \f_e\Big) = \zero \\
    &\Rightarrow \M_e \xdd + \f_f + \f_e \\
    &\ \ \ \ \ - \M_e\left[\frac{\xd\,\xd^\tr}{\xd^\tr\M_e\xd}\right]\Big(\f_f + \big(\f_e - \f_e\big)\Big) = \zero \\
    &\Rightarrow \M_e \xdd + \f_e + \M_e\left[\M_e^{-1} - \frac{\xd\,\xd^\tr}{\xd^\tr\M_e\xd}\right] \f_f = \zero \\
    &\Rightarrow \M_e \xdd + \f_e + \M_e\mR_{\p_e}\big(\M_e\h - \f_e\big) = \zero,
\end{align*}
where we substitute $\h = \M_e\f_f - \f_e$ back in. By Lemma~\ref{lma:EnergyProjection} $\M_e\mR_{\p_e} = \mP_e$, so we get Equation~\ref{eqn:ZeroWorkEnergizationForm}. 
\end{proof}

\begin{lemma} \label{lma:LinearCombinationsUnbiased}
Suppose $\M_e$ is boundary conforming. Then if $\big(\M_e, \f\big)$ and $\big(\M_e, \g\big)$ are both unbiased, then $\big(\M_e, \alpha \f + \beta\g\big)$ is unbiased. 
\end{lemma}
\begin{proof}
Suppose $\x(t)$ is a convergent trajectory with $\x\rightarrow\x_\infty$ If $\x_\infty\in\mathrm{int}(\mathcal{X})$, then $\f\rightarrow \zero$ and $\g\rightarrow \zero$, so $\alpha \f + \beta\g\rightarrow \zero$. If $\x_\infty\in\partial\mathcal{X}$, then $\f = \f_1 + \f_2$ with $\f_1\rightarrow \zero$ and $\f_2\rightarrow \f_\perp\perp\mathcal{T}_{\x_\infty}\partial\mathcal{X}$, and similarly with $\g = \g_1 + \g_2$ with $\g_1\rightarrow \zero$ and $\g_2\rightarrow \g_\perp\perp\mathcal{T}_{\x_\infty}\partial\mathcal{X}$. Then $\alpha \f + \beta \g = \big(\alpha \f_1 + \beta \f_2\big) + \big(\alpha \g_1 + \beta \g_2\big) \rightarrow \alpha \g_1 + \beta \g_2$ which is orthogonal to $\mathcal{T}_{\x_\infty}\partial\mathcal{X}$ since each component in the linear combination is. Therefore, $\alpha \f + \beta \g$ is unbiased.
\end{proof}

For completeness, we state a standard result from linear algebra here (see \cite{gallier2020LinearAlgAndOpt} Proposition 8.7 for a proof). This result will be used in the below lemmas leading up to the proof that energized systems are unbiased.
\begin{proposition}\label{lem:MatrixVectorProductNormBound}
For every matrix $\A$, there exists a contant $c>0$ such that $\|\A\uu\|\leq c\|\uu\|$ for every vector $\uu$, where $\|
\cdot\|$ can be any norm.
\end{proposition}
In the below, for concreteness, we take the matrix norm to be the Frobenius norm.


\begin{lemma} \label{lma:ProjectedUnbiasedIsUnbiased}
Let $\M_e$ be boundary conforming. Then if $\f$ is unbiased with respect to $\M_e$, $\mP_e\f$ is unbiased, where 
$\mP_e = \M_e^{\,\frac{1}{2}}\big[\I - \hat{\vv}\,\hat{\vv}^\tr\big]\M_e^{-\frac{1}{2}}$ is the projection operator defined in Equation~\ref{eqn:EnergyProjector}.
\end{lemma}
\begin{proof}
Let $\x(t)$ be convergent with $\x\rightarrow\x_\infty$ as $t\rightarrow\infty$. If $\x_\infty\in\mathrm{int}(\mathcal{X})$, then since $\mP_e$ is finite on $\mathrm{int}(\mathcal{X})$, by Lemma~\ref{lem:MatrixVectorProductNormBound} there exists a constant $c>0$ such that $\|\mP_e\f\|\leq c\|\f\|$. Since $\|\f\|\rightarrow 0$ as $t\rightarrow 0$, it must be that $\|\mP_e\f\|\rightarrow 0$. 

Consider now the case where $\x_\infty\in\partial\mathcal{X}$. Since $\mP_e$ is a projection operator, for every $\z$ there must be a decomposition into linearly independent components $\z = \z_1 + \z_2$ with $\z_2$ in the kernel such that $\mP_e\z = \mP_e\z_1 + \mP_e\z_2 = \z_1$ (so that $\mP_e^2\z = \mP_e\z_1 \z_1 = \mP_e\z$). Thus, if $\z$ has a property if and only if all elements of a decomposition have that property, then $\mP_e\z = \z_1$ must have that property as well. Specifically, if $\z$ lies in a subspace (respectively, lies orthogonal to a subspace) then $\mP_e\z = \z_1$ lies in that subspace as well (respectively, lies orthogonal to a subspace).

The property of being unbiased is defined by the behavior of the decomposition $\f = \f^\paral + \f^\perp$, where those components are respectively parallel and perpendicular to $\mathcal{T}_{\x_\infty}\partial\mathcal{X}$ in the limit. Following the above outlined subscript convention to characterize the behavior of the projection, we have
\begin{align}
    \mP_e\f 
    &= \mP_e\Big(\f^\paral + \f^\perp\Big)
    = \mP_e\f^\paral + \mP_e\f^\perp\\
    &= \f^{\paral}_1 + \f^{\perp}_1
    \rightarrow \f^{\perp}_1
\end{align}
since $\f$ being unbiased implies $\f^\paral\rightarrow \zero$ and hence $\f^\paral_1\rightarrow \zero$.

Therefore, in both cases, $\mP_e\f$ satisfies the conditions of being unbiased if $\f$ does. 
\end{proof}

\begin{lemma} \label{lma:BoundaryAlignedMetricTransformUnbiased}
Suppose $\M_e(\x,\xd)$ is boundary conforming and boundary aligned. If $\big(\I,\h\big)$ is unbiased then $\big(\M_e, \M_e\h\big)$ is unbiased.
\end{lemma}
\begin{proof}
Let $\x(t)$ be convergent with $\x\rightarrow\x_\infty$ as $t\rightarrow\infty$. If $\x_\infty\in\mathrm{int}(\mathcal{X})$, then since $\M_e$ is finite on $\mathcal{T}\mathrm{int}(\mathcal{X})$, by Lemma~\ref{lem:MatrixVectorProductNormBound} there exists a constant $c>0$ such that $\|\M_e\h\|\leq c\|\h\|$. Since $\|\h\|\rightarrow 0$ as $t\rightarrow 0$, it must be that $\|\M_e\h\|\rightarrow 0$. 

Since $\M_e$ is boundary aligned, there is a subset of Eigenvectors that span the tangent space in the limit as $\x(t)\rightarrow\x_\infty\in\partial\mathcal{X}$. Let $\V_\paral$ denote a matrix containing the Eigenvectors that limit to spanning the tangent space, with $\D_\paral$ a diagonal matrix containing the corresponding Eigenvalues. Likewise, let $\V_\perp$ contain the remaining (perpendicular in the limit) Eigenvectors, with Eigenvalues $\D_\perp$. The metric decomposes as $\M_e = \M_e^\paral + \M_e^\perp$ with $\M_e^\paral = \V_\paral\D_\paral\V_\paral^\tr$ and $\M_e^\perp = \V_\perp\D_\perp\V_\perp^\tr$. Since $\h$ is unbiased, we can express $\h = \h^\paral + \h^\paral = \V_\paral\y_1 + \V_\perp\y_2$, where $\y_1$ and $\y_2$ are coefficients, and $\y_1\rightarrow \zero$ as $t\rightarrow\infty$. Therefore,
\begin{align*}
    \M_e\h &= \Big(\M_e^\paral + \M_e^\perp\Big)\big(\h^\paral + \h^\perp\big)\\
    &= \Big(
        \V_\paral\D_\paral\V_\paral^\tr + \V_\perp\D_\perp\V_\perp^\tr
    \Big) \big(\V_\paral\y_1 + \V_\perp\y_2\big) \\
    &= \V_\paral\D_\paral \y_1 + \V_\perp\D_\perp \y_2 \\
    &\rightarrow \V_\perp\z 
    \ \ \ \perp\ \ \ \mathcal{T}_{\x_\infty}\partial\mathcal{X}
    \ \ \ \mbox{where $\z = \D_\perp \y_2$}.
\end{align*}
Therefore, when $\x_\infty\in\partial\mathcal{X}$, the component parallel to the tangent space vanishes in the limit. 
\end{proof}

\begin{lemma} \label{lma:EnergizationPreservesUnbiasedProperty}
Suppose $\mathcal{S}_\h = \big(\I, \h\big)$ is an unbiased (acceleration) spec and $\Lag_e$ is an unbiased energy Lagrangian with boundary aligned metric $\M_e$. Then $\mathcal{S}_\h^{\Lag_e} = \mathrm{energize}_{\Lag_e}\big\{\mathcal{S}_\h\big\}$ is unbiased.
\end{lemma}
\begin{proof}
$\M_e$ is boundary conforming by hypothesis on $\Lag_e$. Moreover, the energized equation takes the form
\begin{align}
    \M_e\xdd + \f_e + \mP_e\big[\M_e\h - \f_e\big] = \zero
\end{align}
as shown in Proposition~\ref{prop:SystemEnergization}. By Lemma~\ref{lma:BoundaryAlignedMetricTransformUnbiased} $\M_e\h$ is unbiased since $\h$ is unbiased, and likewise $\M_e\h - \f_e$ is unbiased by Lemma~\ref{lma:LinearCombinationsUnbiased} since $\f_e$ is unbiased by hypothesis on $\Lag_e$. That means $\mP_e\big[\M_e\h - \f_e\big]$ is unbiased as well by Lemma~\ref{lma:ProjectedUnbiasedIsUnbiased}. Finally, by again applying Lemma~\ref{lma:LinearCombinationsUnbiased} we see that the entirety of $\f_e + \mP_e\big[\M_e\h - \f_e\big]$ is unbiased. 
\end{proof}
\end{appendices}

\end{document}